\documentclass[10pt]{article}

\usepackage{amsmath,amssymb}
\usepackage{graphicx}% Include figure files
\usepackage[font={small,it}]{caption}
\usepackage[font={small,it}]{subcaption}
\usepackage{color}% Include colors for document elements
\usepackage{hyperref} % hyperref must be loaded before apacite
\usepackage{makecell}
\usepackage{adjustbox}
\usepackage{hhline}
\usepackage{float}
\usepackage[toc,page]{appendix}

\hypersetup{ colorlinks = true, citecolor = black, urlcolor = blue}

\definecolor{background-color}{gray}{0.98}
\newcommand{\br}[1]{\left\{#1\right\}}                            % {#}

\usepackage{hyperref}
\usepackage{blindtext}
\usepackage{comment}
\usepackage{enumerate}
\usepackage{amsmath,amssymb,latexsym,graphicx}
\usepackage{url}
\usepackage{algorithm}% http://ctan.org/pkg/algorithms
\usepackage{color}
\usepackage{amssymb}
\usepackage{mathtools}
\usepackage{graphicx}
\usepackage[algo2e,ruled,vlined,linesnumbered,noend]{algorithm2e}
\usepackage{subcaption}

\usepackage{fullpage}
\graphicspath{ {img/} }
\usepackage{subcaption}
\usepackage{hyperref}
\usepackage{adjustbox}
\usepackage{makecell}

\graphicspath{ {img/} }
\newcommand*{\QEDW}{\hfill\ensuremath{\square}}%
\newcommand{\norm}[1]{\left\lVert#1\right\rVert}

\newcommand{\of}[1]{\left(#1\right)}

\newcommand{\range}{\mathrm{range}}
\newcommand{\cost}{\mathrm{cost_S}}

\DeclareMathOperator*{\argmax}{arg\,max}
\DeclareMathOperator*{\argmin}{arg\,min}

\newcommand{\REAL}{\ensuremath{\mathbb{R}}}

\newcommand{\abs}[1]        {\left| #1\right|}

\renewcommand{\Pr}{\mathrm{pr}}
\newcommand{\pr}{\mathrm{pr}}

\newcommand{\ceil}[1]{\left \lceil #1 \right \rceil}

\newcommand{\smi}{\sum_{i=1}^n}

\newenvironment{proof}{\noindent\normalfont {\bf Proof}.\ }{\QEDW \par\vskip 4mm\par}
 \newtheorem{theorem}{Theorem}[section]
 \newtheorem{corollary}[theorem]{Corollary}
 \newtheorem{lemma}[theorem]{Lemma}
 
\newtheorem{definition}[theorem]{Definition}

 \newtheorem{observation}[theorem]{Observation}

\newcommand{\Y}{\mathcal{X}}

\newcommand{\Z}{\mathbb{X}}
\newcommand{\ff}{f}

\newcommand{\loss}{\mathrm{loss}}

\newcommand{\eps}{\varepsilon}

\newcommand{\smid}{\sum_{i=1}^{d}}
\newcommand{\co}{8}
\DeclarePairedDelimiter\floor{\lfloor}{\rfloor}

\usepackage{fullpage}
\usepackage{setspace}
\usepackage{parskip}
\usepackage{titlesec}
\usepackage[section]{placeins}
\usepackage{xcolor}
\usepackage{breakcites}
\usepackage{lineno}
\usepackage{hyphenat}

\renewenvironment{abstract}
  {{\bfseries\noindent{\abstractname}\par\nobreak}\footnotesize}
  {\bigskip}

\titlespacing{\section}{0pt}{*3}{*1}
\titlespacing{\subsection}{0pt}{*2}{*0.5}
\titlespacing{\subsubsection}{0pt}{*1.5}{0pt}

\usepackage{authblk}

\usepackage{graphicx}
\usepackage[space]{grffile}
\usepackage{latexsym}
\usepackage{textcomp}
\usepackage{longtable}
\usepackage{tabulary}
\usepackage{booktabs,array,multirow}
\usepackage{amsfonts,amsmath,amssymb}
\providecommand\citet{\cite}
\providecommand\citep{\cite}

\newif\iflatexml\latexmlfalse
\providecommand{\tightlist}{\setlength{\itemsep}{0pt}\setlength{\parskip}{0pt}}%

\AtBeginDocument{\DeclareGraphicsExtensions{.pdf,.PDF,.eps,.EPS,.png,.PNG,.tif,.TIF,.jpg,.JPG,.jpeg,.JPEG}}

\usepackage[utf8]{inputenc}
\usepackage[english]{babel}

\usepackage[inline, shortlabels]{enumitem}
\usepackage{float}

\usepackage[margin=1.5in]{geometry}

\newif\ifER
\ERfalse

\begin{document}

\title{Introduction to Coresets: Approximated Mean}

\author{Alaa Maalouf\thanks{Corresponding author: A. Maalouf. E-mail: alaamalouf12@gmail.com.}}
\author{Ibrahim Jubran}
\author{Dan Feldman}
\affil{Robotics \& Big Data Labs,\\
       Computer Science Department,\\
        University of Haifa,\\
        Abba Khoushy Ave 199, Israel.}%

\vspace{-1em}

  \date{}

\begingroup
\let\center\flushleft
\let\endcenter\endflushleft
\maketitle
\endgroup

\selectlanguage{english}
\begin{abstract}
A \emph{strong coreset} for the mean queries of a set $P$ in $\REAL^d$ is a small weighted subset $C\subseteq P$, which provably approximates its sum of squared distances to any center (point) $x\in \REAL^d$. 
A \emph{weak coreset} is (also) a small weighted subset $C$ of $P$, whose mean approximates the mean of $P$. 
While computing the mean of $P$ can be easily computed in linear time, its coreset can be used to solve harder constrained version, and is in the heart of generalizations such as coresets for $k$-means clustering.
In this paper, we survey most of the mean coreset construction techniques, and suggest a unified analysis methodology for providing and explaining classical and modern results including step-by-step proofs. 
In particular, we collected folklore and scattered related results, some of which are not formally stated elsewhere.% but can only be inferred by combining numerous different results. 
Throughout this survey, we present, explain, and prove a set of techniques, reductions, and algorithms very widespread and crucial in this field. However, when put to use in the (relatively simple) mean problem, such techniques are much simpler to grasp.
The survey may help guide new researchers unfamiliar with the field, and introduce them to the very basic foundations of coresets, through a simple, yet fundamental, problem. Experts in this area might appreciate the unified analysis flow, and the comparison table for existing results.
Finally, to encourage and help practitioners and software engineers, we provide full open source code for all presented algorithms.
\end{abstract}%

\sloppy

% \textbf{Remember that you are writing for an interdisciplinary audience.
% Please be sure to discuss interdisciplinary themes, issues, debates,
% etc. where appropriate.} Note that the WIREs are forums for review
% articles, rather than primary literature describing the results of
% original research.

\section{Introduction} \label{sec:intro}
Suppose we wish to build a hospital in some big city with millions of citizens. The challenge which we first face is to compute an optimal location $x^*$ in the city for the hospital, based on the 2-dimensional home address of the citizens, such that: 
\begin{enumerate*}[(i)]
    \item $x^*$ is simultaneously as close as possible to all the citizens of the city, 
    \item $x^*$ is a valid construction site, in terms of restricted zones or occupied spaces, and
    \item it is close to main routs, and (possibly) many more constraints that may depend on the location of the citizens.
\end{enumerate*} 
If we represent each citizen using a point $p \in \REAL^2$ that describes the citizen's house address on the 2D map of the city, then the problem above reduces to finding a vector $x^* \in \REAL^2$ that minimizes the sum of its squared Euclidean distance (SSD) $\norm{p_i-x^*}^2$ to each of the citizens in $P = \br{p_1,\cdots,p_n}$, under some constraints on $x^*$. Formally, we wish to solve
\[
x^* \in \argmin_{x\in \Z} \sum_{i=1}^n\norm{p_i-x}^2,
\]
where $|P|=n$ is the number of citizens, and $\Z$ represents the feasible $x$ locations (which satisfy our constraints). The left hand side then denotes a query from the set of optimal solutions.

In some cases, it is more important for some of the citizens to be closer to the hospital than others. For example, it is probably more important for citizens with health issues to be nearer to the hospital compared to healthy citizens. In this case, we might introduce a weights vector $w = (w_1,\cdots,w_{n})^T \in [0,\infty)^{n}$, where the $i$th entry $w_i$ denotes the importance of the $i$th citizen being close to the hospital. In this case, we wish to compute
\begin{equation} \label{eq:1mean}
x^* \in \argmin_{x\in \Z} \sum_{i=1}^{n}w_i\norm{p_i-x}^2.
\end{equation}

If there are no constraints, i.e., $\Z=\REAL^d$, the solution to~\eqref{eq:1mean} is simply the weighted mean $x^* = \sum_{i=1}^n w_ip_i$ of $P$. This is the reason the problem is called the ``mean'' problem. 

%The problem presented in~\eqref{eq:1mean} is called the \emph{mean} problem. The name ``mean'' stems from the fact that the optimal solution $x^*$ in~\eqref{eq:1mean} when $X = \REAL^d$ (i.e., there are no constraints on the  hospital location) is simply the weighted mean $x^* = \sum_{i=1}^n w_ip_i$ of $P$.

\paragraph{Generalizations.} If we wish to find the optimal location for a new highway (that is represented by a straight line) passing through the city, which is close as possible to all of the citizens, then this is the \emph{line mean} problem, where the goal is to compute a line (or in general, a $j$-dimensional affine subspace of $\REAL^d$) instead of a point, that minimizes the sum of squared Euclidean distances (SSD) to the citizens (points in $\REAL^d$); see e.g.~\cite{maalouf2019fast}. Those problems, where $j\geq1$, are natural generalizations of the mean problem in~\eqref{eq:1mean} where the dimension of the center is $j=0$, and are closely related to the singular value decomposition~\cite{klema1980singular}, principal component analysis (PCA)~\cite{wold1987principal}, linear regression, and many more problems in machine learning and computational geometry.
To this end, this works, which is the second work in a series of surveys to come, solely handles the mean problem, which we believe is the foundation for all the complex optimization problems mentioned above.

%In more complex cases, rather than a point $x^*$, we may wish to find a \emph{line} (or in general, a $j$-dimensional affine subspace of $\REAL^d$) that minimizes the SSD to the citizens (points in $\REAL^d$); see e.g.~\cite{maalouf2019fast}. For example. If we wish to find the optimal location for a new highway (straight line) passing through the city, which is close as possible to all our citizens. Those problems, where $j\geq1$, are a natural generalization of the mean problem in~\eqref{eq:1mean} where $j=0$, and are closely related to the Singular Value Decomposition~\cite{klema1980singular}, Principal Component Analysis (PCA)~\cite{wold1987principal}, and many more problems from machine learning and computational geometry.
%To this end, this works, which is the second work in a series of surveys to come, solely handles the mean problem, which we believe is the foundation for all the complex optimization problems mentioned above.

With the increase of input points (data size) and data acquisition models, the computational load for finding the best location for a facility under some constraints, or computing the SVD of the data matrix becomes computationally infeasible in real-time, especially when applied on small IoT devices and huge databases. Beside computing the optimal location, we might need to evaluate, in real-time, the cost of some given potential hospital location.
To mitigate the above problems and handle simultaneously various data models, we suggest to utilize a relatively new paradigm called a \emph{coreset}, which has been gaining popularity in the past decade. This paradigm suggests to provably summarize the input data instead of improving the existing solvers, while allowing to evaluate any query (potential location), on the compressed data.

\paragraph{A coreset }is a modern problem-dependent data summarization scheme that approximates the original (big) dataset $P$ in some provable sense with respect to a (usually infinite) set of questions / queries $\Y$ defined by the problem at hand, and an objective loss/cost function $\ff_\loss$.
If we indeed succeed to provably compress the data, in this sense, then we can then compute the query that minimizes the given objective cost function on the compressed coreset instead of the original data, thus saving time, energy and space without compromising the accuracy by more than a small multiplicative factor.

Coresets are especially useful for learning big data since an off-line and possibly inefficient coreset construction for ``small data" implies constructions that maintains coreset for streaming, dynamic (including deletions) and distributed data in parallel. This is via a simple and easy to implement framework that is often called merge-reduce trees; see~\cite{bentley1980decomposable, indyk2014composable,agarwal2013mergeable}.
The fact that a strong coreset approximates every query (and not just the optimal one for some criterion) implies that we may solve hard optimization problems with non-trivial and non-convex constraints by running a possibly inefficient algorithm such as exhaustive search on the coreset, or running existing heuristics numerous times on the small coreset instead of once on the original data. Similarly, parameter tuning, model selection, or cross validation can be applied on a coreset that is computed once for the original data as explained in~\cite{maalouf2019fast}.

In recent years, coresets were applied to many machine learning algorithms e.g. logistic regression~\cite{huggins2016coresets,munteanu2018coresets,karnin2019discrepancy}, \emph{SVM}~\cite{har2007maximum,tsang2006generalized,tsang2005core,tsang2005very,tukan2020coresets}, clustering problems~\cite{feldman2011scalable,gu2012coreset,jubran2020sets,bachem2018one,lucic2015strong, schmidt2019fair,sohler2018strong}, matrix approximation~\cite{feldman2013turning, maalouf2019fast,feldman2010coresets,sarlos2006improved,maalouf2020faster}, $\ell_z$-regression~\cite{cohen2015lp, dasgupta2009sampling, sohler2011subspace}, and others~\cite{huang2021novel,cohen2021improving,huang2020coresets,mirzasoleiman2020coresets}; see surveys~\cite{feldman2020core,phillips2016coresets,jubran2019introduction}. % and showed very useful in various application.

There are many types of coresets and coreset constructions. In this survey we focus on what is sometimes called strong and weak coresets. Informally, a \emph{strong coreset} $C$ guarantees that for \emph{every} query $x\in \Y$, the value of the cost function $\ff_\loss(C,x)$ when applied on the coreset $C$ is approximately the same as $\ff_\loss(P,x)$ when applied on the original full data; see e.g.~\cite{lucic2015strong,feng2019strong}.

A \emph{weak coreset} $C$ (usually) only guarantees that the optimal query for $C$, $x_C^* \in \argmin_{x\in \Y} \ff_\loss(C,x)$ and the optimal query for $P$, $x_P^* \in \argmin_{x\in \Y} \ff_\loss(P,x)$ yield approximately the same cost on the full data, i.e., $\ff_\loss(P,x_P^*) \sim \ff_\loss(P,x_C^*)$; see e.g.~\cite{feldman2007ptas}.

Since a strong coreset $C$ approximates every query $x\in \Y$, we can also minimize $\ff_\loss$ on $C$ given further (previously unknown) constraints $\Z$, since the query set under the assumptions $\Z \cap \Y \subseteq \Y$ is also contained in $\Y$.

We would usually prefer to compute a coreset $C$ which is a (possibly weighted) \emph{subset} $C \subseteq P$ of the input $P$. Such a \emph{subset coreset} has multiple advantages over a non-subset coreset, which are 
\begin{enumerate*}[(i)]
    \item preserved sparsity of the input,
    \item interpretable, 
    \item may be used (heuristically) for other problems, and
    \item less numerical issues that occur when non-exact linear combination of points are used. 
\end{enumerate*}
Unfortunately, not all problems admit such a subset coreset. For further discussion and examples see e.g.~\cite{maalouf2020tight,feldman2020core}.

\paragraph{Why coreset for the mean problem?}
While the mean problem in~\eqref{eq:1mean} is a relatively simple problem, it lies at the basis of more involved and very common problems in machine learning, e.g., the classic $k$-means clustering. In particular, we can always improve a given $k$-clustering, by replacing the center of each cluster by its mean (if this is not already the case). This is indeed the idea behind the classic Lloyd's heuristic~\cite{lloyd1982least} and also behind some coresets for $k$-means~\cite{barger2020deterministic}. 
Most coreset construction algorithms for those hard problems usually borrow or generalize tricks and techniques used in coreset constructions for the (simpler) mean problem.
Furthermore, other works, which seem unrelated at first glance, require at their foundations an algorithm for computing a mean coreset. 
Such problems include coresets for Kernel Density Estimates (KDE) of Euclidean kernels~\cite{phillips2020near}, least squares problems, e.g., coresets for linear regression and the singular value decomposition~\cite{maalouf2019fast}, and coresets for signals~\cite{rosman2014coresets}. 
For example in~\cite{maalouf2019fast} it was shown that in order to compute a lossless SVD (or linear regression) coreset for an $n\times d$ matrix $A$, it is sufficient to compute a smaller $m\times d$ matrix $C$ such that $A^TA=C^TC$. 

The scatter matrix $A^TA$ of an input matrix $A = (a_1 \mid \cdots \mid a_n)^T \in \REAL^{n\times d}$ is the sum $A^TA = \sum_{i=1}^n a_ia_i^T$ over $n$ $d\times d$ matrices. Each such matrix can be "flatten" to a vector in $\REAL^{d^2}$. Hence, we can compute a smaller subset of these $d^2$-dimensional vectors, which accurately estimate their original sum. We thus obtain a weighted subset of the rows of $A$ whose scatter matrix is the desired $A^TA$, with no additional error~\cite{jubran2019introduction}. 
A coreset that introduces multiplicative $1+\eps$ error for this problem (SVD/linear regression) was suggested in~\cite{feldman2016dimensionality}, also here the authors suggested a reduction to the problem of computing a mean coreset with multiplicative $1+\eps$ error for a set of point in a higher dimensional space.
Another example is in the context of $k$-means, where~\cite{barger2020deterministic} showed that in order to compute a $k$-means coreset for a set of points $P$ it is suffices to cluster these points to a large number of clusters, and compute a mean coreset for each cluster, then take the union of these coresets to a single unite set, which is proven to be a $k$-means coreset for $P$.

Not only are the mean-related results scattered across numerous papers and books dating from the last century and till today, but some of those constructions and proofs are not formally stated elsewhere, and can only be inferred by combining many different results. 

\paragraph{Main goal. }To this end, in this work we aim to review the wide range of techniques and methodologies behind the constructions of mean coresets, ranging from loss-less to lossy, from deterministic to randomized, and from greedy to non-greedy constructions. Examples include accurate coresets via computational geometry, random sampling-based coresets via Bernstein inequality, and greedy deterministic coresets via the Frank-Wolfe algorithm~\cite{clarkson2010coresets}.
We provide in-depth proofs, under a unified notation, for all the suggested approaches, and guide the reader through them.
We also analyze and compare all the presented results based on their construction time, size, and the properties discussed above; see Table~\ref{table:ourContrib}.
Both to help readers outside the theoretical computer science community, and to encourage the usage and generalization of the presented algorithms, we provide full open source code for all the presented results~\cite{opencode}.

\begin{table}[ht]
\centering
\begin{adjustbox}{width=\textwidth}
\small
\begin{tabular}{ | c | c | c | c | c | c | c | c |}
\hline
\textbf{\makecell{Coreset\\type}} & \textbf{\makecell{Input\\weights}} & \textbf{\makecell{Probability\\of failure}} & \textbf{\makecell{Multiplicative\\error}} & \textbf{Coreset size $|c|$} & \textbf{Properties} & \textbf{\makecell{Construction\\time}} &  \textbf{\makecell{Formal\\statement}}\\
\hline
Strong & \makecell{$w\in(0,\infty)^n$} & $\delta = 0$ & $\eps = 0$ & $O(1)$ & \makecell{Not a subset,\\requires a different\\cost function} & $O(nd)$ & Section~\ref{sec:Accurate1mean}\\
\hline
Strong & \makecell{$w\in(0,\infty)^n$} & $\delta = 0$ & $\eps = 0$ & $d+2$ & \makecell{Subset\\$u\in\REAL^n$\\$\smi u_i = \smi w_i$} & $O(nd^2)$ & Section~\ref{sec:Accurate1mean}\\
\hline
Strong & \makecell{$w\in(0,\infty)^n$} & $\delta = 0$ & $\eps = 0$ & $d+3$ & \makecell{Subset\\$u\in[0,\sum_{p\in P}w(p)]^n$\\$\smi u_i = \smi w_i$} & $O(nd+d^4\log{n})$ & Section~\ref{sec:Accurate1mean}\\
\hline
\hline
Strong & $w \equiv \frac{1}{n}$ & $\delta \in (0,1)$ & $\eps \in (0,1)$ & $O\left(\frac{d+\log(1/\delta)}{\eps^2}\right)$ & \makecell{Subset\\$u\in\REAL^n$} & $O(nd)$ & Lemma~\ref{tightSens}\\
\hline %weak-probabilistic-coreset-theorem
Weak & $w \equiv \frac{1}{n}$ & $\delta \in (0,1)$ & $\eps \in (0,1)$ & $O\left(\frac{d+\log(1/\delta)}{\eps}\right)$ &  \makecell{Subset\\$u\in\REAL^n$}  & $O(nd)$ & Lemma~\ref{tightSensWeak}\\
\hline
Strong & $w\in(0,\infty)^n$ & $\delta \in (0,1)$ & $\eps \in (0,1)$ & $O\left(\frac{\log(d/\delta)}{\eps^2}\right)$ & \makecell{Subset\\$u\in\REAL^n$} & $O(nd)$ & Theorem~\ref{strong-coreset-theorem-bern}\\
\hline
Weak & $w\in(0,\infty)^n$ & $\delta \in (0,1)$ & $\eps \in (0,1)$ & $O\left(\frac{\log(d/\delta)}{\eps}\right)$ & \makecell{Subset\\$u\in\REAL^n$} & $O(nd)$ & Theorem~\ref{weak-coreset-theorem-bern}\\
\hline
\hline
Strong & $w\in(0,\infty)^n$ & $\delta =0$ & $\eps \in (0,1)$ & $O\left(\frac{1}{\eps^2}\right)$ & \makecell{Subset\\$u\in[0,(1+\eps)\smi w_i]^n$\\ $\abs{\smi w_i -\smi u_i}\leq \eps \smi w_i $} & $O\left(\frac{nd}{\eps^2}\right)$ & Theorem~\ref{strong-coreset-theorem}\\
\hline
Weak & $w\in(0,\infty)^n$ & $\delta =0$ & $\eps \in (0,1)$ & $O\left(\frac{1}{\eps}\right)$ & \makecell{Subset\\$u\in[0,(1+\sqrt{\eps})\smi w_i]^n$\\ $\abs{\smi w_i -\smi u_i}\leq \sqrt{\eps} \smi w_i $} & $O\left(\frac{nd}{\eps}\right)$ & Theorem~\ref{weak-coreset-theorem}\\
\hline
\hline
Weak & $w \equiv \frac{1}{n}$ & $\delta \in (0,1)$ & $\eps \in (0,1)$ & $O\left(\frac{1}{\eps \delta}\right)$ &  \makecell{Subset\\$u\in[0,1]^n$\\$\smi u_i = 1$}  & $O(\frac{1}{\eps \delta})$ & Lemma~\ref{weak_coreset_markov}\\
\hline %weak-probabilistic-coreset-theorem
Weak & $w \equiv \frac{1}{n}$ & $\delta \in (0,1)$ & $\eps \in (0,1)$ & $O\left(\frac{1}{\eps}\right)$ &  \makecell{Subset\\$u\in[0,1]^n$\\$\smi u_i = 1$}  & $O\left(d\cdot\left(\log^2(\frac{1}{\delta}) +\frac{\log(\frac{1}{\delta})}{\eps}\right)\right)$ & Lemma~\ref{weak-probabilistic-coreset-theorem}\\
\hline
\end{tabular}
\end{adjustbox}
\caption{\textbf{Summary of mean coresets.} This table presents various coresets for the mean problem, both $\eps$-coresets from this work and accurate coresets from~\cite{jubran2019introduction}. The input for all algorithms is a set $P = \br{p_1,\cdots,p_n} \subseteq \REAL^d$ and a weights vector $w \in \REAL^n$. See Section~\ref{sec:intro} and Definitions~\ref{def:strongeCoreset} and~\ref{def:weakCoreset} for the different coreset types. $\delta$ represents the probability of failure of the corresponding algorithm, i.e., a deterministic algorithm has $\delta=0$. The measured error $\eps$ is a multiplicative error, i.e., $\eps = \displaystyle\argmax_{x\in \REAL^d}\frac{\ff_\loss((P,w),x) - \ff_\loss((P,u),x)}{\ff_\loss((P,w),x)}$. %Rows marked with a $\star$ present our new novel state of the art results.
}
\label{table:ourContrib}
\end{table}

\section{Paper Overview}
This survey is part of a series of surveys that aim to give introduction to coresets; see~\cite{feldman2019core} and~\cite{jubran2019introduction}. 
This work is organized as follows. We first introduce the notations and definitions in Section~\ref{sec:notations}.
In Section~\ref{sec:Accurate1mean}, we briefly summarize a first type of mean coresets constructions. Those coresets are often called \emph{accurate coresets}, as they do not introduce any error when compressing the data, unlike most of the other coresets when such an approximation error is unavoidable for obtaining a small ($o(n)$) coreset.
In Section~\ref{sec:probred} we present a reduction between the problem of computing a (strong and weak) mean coreset for an arbitrary set of input point $Q \subset \REAL^d$ to the problem of computing a mean coreset to a corresponding, yet much simpler, set of points $P$ which we call a ``normalized weighted set''. This set satisfies a set of properties (e.g., zero mean) that will simplify the analysis later on; see Observation~\ref{to_0_mean} and Corollary~\ref{cor:reduction}. 

In Sections~\ref{sec:strong-coreset-red} and~\ref{sec:weak-coreset-red}, we continue and simplify the definition of coreset for a normalized weighted set $P$ by explaining what (sufficient) properties should hold for a set $C$ in order to be a strong/weak coreset for $P$. 
Through Section~\ref{sec:strondelta}, we show how to compute, with a high probability, a (strong and weak) coreset for such a normalized set $P$ based on two different approaches: 
\begin{enumerate*}[(i)]
\item in Subsection~\ref{sec:sensitivityCoresets} we present a random coreset construction which utilizes the well known sensitivity sampling framework~\cite{DBLP:journals/corr/BravermanFL16,feldman2011unified}, 
\item then in Section~\ref{sec:bernstein} we show how to utilize the Bernstein inequality to obtain smaller coresets in the same running time. 
\end{enumerate*}
The two approaches above are very similar, and basically differ in their analysis.
We then present, in Section~\ref{sec:detcore}, a deterministic coreset construction algorithm (zero probability of failure) for an input normalized weighted set; see Theorem~\ref{weak-coreset-theorem} and Theorem~\ref{strong-coreset-theorem}. The main technique in this section is to normalize the data in a way that enables the use of the classic Frank-Wolfe algorithm~\cite{frank1956algorithm} from 1956 (that was re-discovered only recently by~\cite{clarkson2010coresets}). 
Finally, in Section~\ref{sec:weaksublinear}, we present two algorithms for computing, with high probability, a weak coreset in time that is sublinear in the input size. Table~\ref{table:ourContrib} summarizes all the results that are written in this paper.

\section{Notations and assumptions}\label{sec:notations}

In this section we first we first introduce our notations that will be used through the paper, and then give our main definitions.

\begin{paragraph}{Notations.} For a pair of integers $d,n\geq 1$, we denote by $\REAL^{n\times d}$ the union over every $n\times d$ real matrix and $[n] = \br{1,\cdots, n}$. The $\ell_2$, $\ell_1$ and $\ell_0$ norm of a vector $v=(v_1,\cdots,v_d)\in\REAL^d$ are denoted, respectively, by $\norm{v}=\sqrt{\smid v_i^2}$, $\norm{v}_1 =\smid \abs{v_i}$, and $\norm{v}_0$, where and $\norm{v}_0$ is the number of non-zero entries in $v$.
%For a vector $v=(v_1,\cdots,v_d)\in\REAL^d$ we use $\norm{v}=\sqrt{\smid v_i^2}$ to denote its $2$-norm, $\norm{v}_1 =\smid \abs{v_i}$ to denote its $1$-norm, and the $0$-norm is denoted by $\norm{v}_0$ and its equal to the number of non-zero entries in $v$. We denote by $\mathbf{0}_d = (0,\cdots,0)^T \in \REAL^d$ the $d$-dimensional zero vector.

For a matrix $A\in\REAL^{n\times d}$ the \emph{Frobenius norm} $\norm{A}_F$ is the squared root of its sum of squared entries, and $tr(A)$ denotes its trace. A vector $w\in [0,1]^n$ is called a \emph{distribution vector} if its entries sum up to one. For an event $B$ we use $\Pr(B)$ as the probability that event $B$ occurs.

%There are many different notations used to assign weights for a set of points. One common notation is the \emph{weights function} notation, where for a set of points $P = \br{p_1,\cdots,p_n} \subseteq \REAL^d$, a function $w:P\to \REAL$ assigns a weight $w(p_i)$ for every $p_i\in P$. Another notation is the \emph{weights vector} notation, where $w = (w_1,\cdots,w_n)^T \in \REAL^n$ such that $w_i$ is the weight of the point $p_i \in P$. In this paper, we use the later notation of a weights vector. For a weights vector $w \in \REAL^n$ we use $w_i$ to denote the $i$th entry of $w$ for every $i \in [n]$.

A \textit{weighted set} is a pair $P'=(P,w)$ where $P=\br{p_1,\cdots,p_n} \subseteq \REAL^d$ is a set of $n$ \emph{points}, and $w=(w_1,\cdots,w_n)^T \in (0,\infty)^n$ is called a \emph{weights vector} that assigns every $p_i\in P$ a weight $w_i \in \REAL$. The size of $P'$ is $|P| = n$ and the cardinality of $P'$ is the number of non zero entries $\norm{w}_0$ of $w$. Finally the \emph{weighted sum} of a weighted set $(P,w)$ is defined as $\smi w_ip_i$, and its weighted mean is $\smi \frac{w_i}{\norm{w}_1}p_i$.
%A \emph{weighted point} is a weighted set of size $|P|=1$.
%Finally, a weighted set $(P,\mathbf{1})$ where $\mathbf{1}$ is the weight function $w=(1,\cdots,1)^T$ that assigns $w_i=1$ for every $p_i\in P$ may be denoted by $P$ for short.

%For an integer $k \in \br{0,\cdots,d-1}$, a \emph{$k$-subspace} is a shorthand for a $k$-dimensional linear (non-affine) subspace of $\REAL^d$ (i.e., it contains the origin). An \emph{affine $k$-subspace} ($k$-flat) is a translation of a $k$-subspace, i.e.,  that may not contain the origin.
%For every point $p\in \REAL^d$ and an affine $k$-subspace $S$ of $\REAL^d$, we define $\mathrm{proj}(p,S)=\argmin_{x\in S} \norm{p-x}_2$ to be the projection of the point $p$ onto the affine $k$-subspace $S$ and $\dist(p,S)=\min_{x\in S}\norm{p-x}_2=\norm{p-\mathrm{proj}(p,S)}_2$ to be the Euclidean distance between the point $p$ to its closest point on $S$. This distance to the power of $z\geq1$ is denoted by $D_z(p,S) =\dist^z(p,S)$ and for brevity, we define $D(p,S) =D_2(p,S) =\dist^2(p,S)$.
%A \emph{universal constant} is a real number that is independent of any other variable, say $3$ or $0.1$.
\end{paragraph}
%In this research we focus on the one mean problem, we show variant ways of computing a Corset with provable guarantees for this problem. Mainly we prove most of our lemmas and theorem on a weighted set $(P,w)$ where $P=\br{p_1,\cdots,p_n}$ of $n$ points in $\REAL^d$ and $w$ is a distribution vector such that the weighted mean of $(P,w)$ is equal to zero and the weighted variance is equal to one, we rely on Oservation 4.1~\cite{AlaamFeldmanAdiel} which says that proving a coreset on such set yields proving a coreset any set.
\begin{definition} [Normalized weighted set]\label{defNormalize}
A \emph{normalized weighted set} is a weighted set $(P,w)=(\br{p_1,\cdots,p_n},(w_1,\cdots,w_n)^T)$ that satisfies the following three properties:%where $P=\br{p_1,\cdots,p_n} \subseteq \REAL^d$ and $w=(w_1,\cdots,w_n)^T\in(0,\infty)^n$ 

\begin{enumerate}[(a)]
\item Weights sum to one: $\smi w_i = 1$, \label{a}
\item The weighted sum is the origin: $\smi w_ip_i=\mathbf{0}_d$, and\label{b}
\item The weighted sum of squared norms is $1$: $\smi w_i\norm{p_i}^2=1$. \label{c}
\end{enumerate}
\end{definition}

%We borrow the \emph{query space} definition from~\cite{jubran2019introduction} for our specific one mean problem as follows.
%\begin{definition}[query space]\label{def::query space}
%Let $\Y = \REAL^d$ be a (possibly infinite) set called \emph{query set},  $P'=(P,w)$ be a weighted set called the \emph{input set}, $\ff:P\times \Y\to [0,\infty)$ be called a \emph{cost function} such that $\ff(p,x) = (p-x)$ for every $p\in P$ and $x\in \REAL^d$, and $\loss = \norm{\cdot}^2$ be a function that returns the squared $\ell_2$ norm of every real vector.
%The tuple $(P,w,\Y,\ff,\loss)$ is called a \emph{query space}. 

%For every weighted set $C'=(C,u)$ such that $C=\br{c_1,\cdots,c_m}$,
%and every $x\in \REAL^d$ we define the overall fitting error of $C'$ to $x$ by
%\[
%\ff_{\loss}(C',x):=\loss(u_1\ff(c_1,x),\cdots,u_m\ff(c_m,x)) = \sum_{i=1}^m %u_i\norm{c_i-x}^2.
%\]
%Therefore, $\ff_{\loss}(C',x)$ is the weighted sum of squared distances between the points of $C$ to $x$.
%\end{definition}

In what follows is the definition of a strong $\eps$-coreset for the mean problem. A coreset for a weighted set $(P,w)$ is nothing but a re-weighting of the points in $P$ by a new weights vector $u$, such that \emph{every} query $x\in \REAL^d$ will yield approximately the same cost when applied to either $(P,w)$ or $(P,u)$. We usually aim to compute a weighted set $(P,u)$ of cardinality $\norm{u}_0 \ll \norm{w}_0$.
\begin{definition} [Strong mean $(\eps,\delta)$-coreset] \label{def:strongeCoreset}
Let $(P,u)$ and $(P,w)$ be two weighted sets in $\REAL^d$ such that $|P| = n$, and let $\eps, \delta \in [0,1)$. 
We say that $(P,u)$ is a \emph{strong mean $(\eps,\delta)$-coreset} for $(P,w)$ of cardinality $\norm{u}_0$ if, with probability at least $1-\delta$, for every $x\in \REAL^d$,
\[
 \left|\smi w_i\norm{p_i-x}^2 - \smi u_i\norm{p_i-x}^2 \right| \leq \eps \smi w_i\norm{p_i-x}^2.
\]
If $\eps = 0$, we say that $(P,u)$ is a strong mean \emph{accurate coreset} for $(P,w)$, and if $\delta = 0$ we say that the coreset is deterministic and simply call it a strong $\eps$-coreset.
\end{definition}

A weak mean $\eps$-coreset for $(P,w)$ is a weighted set $(P,u)$ such that solving for the optimal query $x\in\REAL^d$ on the coreset $(P,u)$ and applying it on $(P,w)$ yields approximately the same result as if computing the optimal solution of the original set $(P,w)$.
\begin{definition} [Weak mean $(\eps,\delta)$-coreset] \label{def:weakCoreset}
Let $(P,u)$ and $(P,w)$ be a pair of weighted sets, and let $n=|P|$. Let $\bar{p} \in \argmin_{x\in \REAL^d}\sum_{i=1}^{n}w_i\norm{p_i-x}^2$, $\bar{s} \in \argmin_{x\in \REAL^d} \sum_{i=1}^{n}u_i\norm{p_i-x}^2,$ and put $\eps, \delta \in [0,1)$.
Then $(P,u)$ is a \emph{weak mean $(\eps,\delta)$-coreset} (or weak $(\eps,\delta)$-coreset in short) for $(P,w)$ of cardinality $\norm{u}_0$ if with probability at least $1-\delta$, we have:
\[
 \left| \smi w_i\norm{p_i-\bar{p}}^2 - \smi w_i\norm{p_i-\bar{s}}^2\right| \leq \eps \cdot \smi w_i\norm{p_i-\bar{p}}^2.
\]
If $\eps = 0$, we say that $(P,u)$ is a weak mean \emph{accurate coreset} for $(P,w)$, and if $\delta = 0$ we say that the coreset is deterministic and simply call it a weak $\eps$-coreset.
\end{definition}

\section{Accurate mean coresets} \label{sec:Accurate1mean}
Before going into the more involved coresets for the mean problem, in this section we will briefly summarize the most simple coresets which are the accurate coresets; see Definition~\ref{def:strongeCoreset}. Those coresets do not introduce any error when compressing the data, i.e., $\eps = 0$.
The coresets presented in this section are explained in detail in~\cite{jubran2019introduction}.

Let $(P,w)$ be a weighted set (input set) of size $|P| = n$ and $x\in \REAL^d$ be a vector (query).
\paragraph{Simple statistics. }We first make the following simple observation:
\begin{equation} \label{eqStatistics}
\smi w_i\norm{p_i-x}^2=\smi w_i\norm{p_i}^2 -2x^T\smi w_ip_i + \norm{x}^2\smi w_i.
\end{equation}
By this observation, we notice that $\smi w_i\norm{p_i-x}^2$ is equal to the sum of the following $3$ terms: (i) $\smi w_i\norm{p_i}^2$, (ii) $-2x^T\smi w_ip_i$, and (iii) $\norm{x}^2\smi w_i$. Notice that the first term is independent of the query $x$, the second depends on $\norm{x}^2$ and $\smi w_i$, the third term depends on $x^T$ and $\smi w_ip_i$.
Therefore, by pre-computing in $O(n)$ time and storing in memory the following statistics: $\smi w_i\norm{p_i}^2$, $\smi w_ip_i$, and $\smi w_i$, for any (new) given query $x$ we can evaluate $\smi w_i\norm{p_i-x}^2$ in $O(1)$ time by simply evaluating the $3$ terms from~\eqref{eqStatistics} using $x$ and the stored statistics.

We note that this ``coreset'' is different than other coresets presented in this paper, since it is not a subseteq of the input and requires evaluating a different cost function on the coreset than on the original data.

\paragraph{Subset coreset. } We now aim to compute a mean coreset $(P,u)$ of cardinality $\norm{u}_0 << n$. From~\eqref{eqStatistics}, we know that if a weighted set $(P,u)$ satisfies: (i) $\smi w_i\norm{p_i}^2 = \smi u_i\norm{p_i}^2$, (ii) $\smi w_ip_i = \smi u_ip_i$, and (iii) $\smi w_i = \smi u_i$, then clearly $ \smi u_i \norm{p_i-x}_2^2 = \smi w_i \norm{p_i-x}_2^2$ for every $x\in \REAL^d$. Therefore, to compute an accurate strong mean coreset $(P,u)$ for $(P,w)$, we simply need to ensure that (i)--(iii) holds.

It turns out that we can compute in $O(nd^2)$ time a coreset $(P,u)$ of cardinality $\norm{u}_0 \leq d+2$ where $u \in \REAL^{n}$ that satisfies the conditions above. Furthermore, if the input weights $w$ are non-negative, i.e., $w\in [0,\infty)^n$, we can compute in $O(nd + d^4\log{n})$ a coreset $(P,u)$ of cardinality $\norm{u}_0 \leq d+3$ where $u \in [0,\sum_{p\in P}w(p)]^{d+3}$ is both non-negative and bounded; see full details in~\cite{jubran2019introduction}.

\section{Problem Reduction for $\eps$-Coresets}\label{sec:probred}
In this section, we argue that in order to compute a strong (weak) $\eps$-coreset for an input weighted set $(Q,m)$, it suffices to compute a strong (weak) $\eps$-coreset for its corresponding normalized (and much simpler) weighted set $(P,w)$ as in Definition~\ref{defNormalize}; see Corollary~\ref{cor:reduction}.

Note that we do not actually normalize the given input data. The normalization is used only in the analysis and coresets proof of correctness.
\subsection{Reduction To Normalized Weighted Set}

\observation \label{to_0_mean} Let $Q=\{q_1,\cdots,q_n\}$ be a set of $n\geq2$ points in $\REAL^d$, $m\in(0,\infty)^n$, $w \in(0,1]^n$ be a distribution vector such that $w=\frac{m}{\norm{m}_1}$, %for every $j\in [n]$ we have $w_j=\frac{m_j}{\norm{m}_1 }$.
 ${\mu=\smi{w_i q_i}}$ and $\sigma=\sqrt{\smi w_i\norm{q_i - \mu}^2 }$.
Let $P=\br{p_1,\cdots,p_n}$ be a set of $n$ points in $\REAL^d$, such that for every $j\in [n]$ we have $p_j=\frac{q_j - \mu}{\sigma}$. Then, $(P,w)$ is the corresponding normalized weighted set of $(Q,m)$, i.e., (i)-(iii) hold as follows:\label{obs1}
\begin{enumerate}[(i)]
\item $\smi w_i = 1$, \label{sumw_it}
\item $\smi w_ip_i=\mathbf{0}$, and \label{first_it}
\item $\smi w_i\norm{p_i}^2=1$. \label{sec_it}
%\item For every $j\in [n], x\in\mathbb{R}^d$ and $y=\frac {x-\mu}{\sigma}$, we have\\ \label{third_it}
	%$$\sum_{i=1}^n w_i||p_i-y||^2 = \frac{1}{\sigma^2 \norm{m}_1} \sum_{i=1}^n m_i||q_i-x||^2.$$
%	$$w_j\norm{p_j - y}^2  = \frac{m_j}{\sigma^2 \norm{m}_1}\norm{q_j - x}^2 .$$
\end{enumerate}
\begin{proof}
\begin{align*}
&{\textit{~\ref{sumw_it} }} \smi w_i = 1\text{ immediately holds by the definition of $w$.}\\
&{\textit{~\ref{first_it} }}
\sum_{i=1}^n w_ip_i =  \sum_{i=1}^n w_i \cdot \frac{q_i - \mu} {\sigma} = \frac{1}{\sigma} \left( \sum_{i=1}^n w_iq_i -  \sum_{i=1}^nw_i\mu \right) = \frac{1}{\sigma} \left( \mu -  \sum_{i=1}^nw_i \mu   \right) = \frac{1}{\sigma}  \mu \left(1  -  \sum_{i=1}^nw_i \right) = 0,
%\end{align*}
\intertext{where the first equality holds by the definition of $p_i$, the third holds by the definition of $\mu$, and the last is since $w$ is a distribution vector.}
%\begin{align*}
&{\textit{~\ref{sec_it} }}
\smi w_i\norm{p_i}^2 = \smi w_i\norm{ \frac{q_i - \mu}{\sigma} }^2=\frac{1}{\sigma^2}\smi w_i\norm{{q_i - \mu}}^2 =\frac{\smi w_i\norm{{q_i - \mu}}^2}{\smi w_i\norm{q_i - \mu}^2} =1,
\end{align*}
where the first and third equality hold by the definition of $p_i$ and $\sigma$, respectively.
\end{proof}

\begin{corollary} \label{cor:reduction}
Let $(Q,m)$ be a weighted set, and let $(P,w)$ be its corresponding normalized weighted set as computed in Observation~\ref{to_0_mean}. 
Let $(P,u)$ be a strong (weak) $\eps$-coreset for $(P,w)$ and let $u' = \norm{m}_1 \cdot u$. Then $(Q,u')$ is a strong (weak) $\eps$-coreset for $(Q,m)$.
\end{corollary}
\begin{proof}
Put $x \in \REAL^d$ and let $y=\frac{x - \mu}{\sigma}$. Now, for every $j\in [n]$, we have that
\begin{align}\label{keyobs}
\norm{q_j - x}^2 
=\norm{ \sigma p_j + \mu- (\sigma y + \mu) }^2
= \norm{ \sigma p_j  - \sigma y }^2
= \sigma^2 ||p_j-y||^2,
\end{align}
{where the first equality is by the definition of $y$ and $p_j$.}

We prove Corollary~\ref{cor:reduction} first for the case of a strong $\eps$-coreset, and then for the case of a weak $\eps$-coreset.

\noindent\textbf{Proof for a strong $\eps$-coreset.}
Let $(P,u)$ be a strong $\eps$-coreset for $(P,w)$. We prove that $(Q,u')$ is a strong $\eps$-coreset for $(Q,m)$. Observe that
\begin{align} \label{firstoneneeded}
\abs{\smi (m_i-u'_i)\norm{q_i-x}^2 } =
\abs{\smi (m_i-u'_i) \sigma^2 \norm{p_i-y}^2 } = 
\abs{\smi \norm{m}_1   \sigma^2 (w_i-u_i) \norm{p_i-y}^2 },
\end{align}
where the first equality holds by~\eqref{keyobs}, and the second holds by the definition of $w$ and $u'$.

Since $(P,u)$ is a strong $\eps$-coreset for $(P,w)$
\begin{align} \label{secondoneneeded}
\abs{\smi \norm{m}_1  \sigma^2 (w_i-u_i) \norm{p_i-y}^2}
\leq\eps {\smi \norm{m}_1  \sigma^2 w_i \norm{p_i-y}^2}
= \eps {\smi m_i \norm{q_i-x}^2 },  
\end{align}
where the equality holds by~\eqref{keyobs} and since $ w=\frac{m}{\norm{m}_1} $.

The proof for the case of a strong $\eps$-coreset concludes by combining~\eqref{firstoneneeded} and~\eqref{secondoneneeded} as
$$\abs{\smi (m_i-u'_i)\norm{q_i-x}^2 }\leq \eps \smi m_i \norm{q_i-x}^2.$$
\newcommand{\mut}{\tilde{\mu}}

\noindent\textbf{Proof for a weak $\eps$-coreset.}
Let $(P,u)$ be a weak $\eps$-coreset for $(P,w)$. We prove that $(Q,u')$ is a weak $\eps$-coreset for $(Q,m)$. First, we observe the following equalities
\begin{equation}\label{reductiontop}
\begin{split}
&\smi m_i \norm{q_i-\frac{\smi u_i q_i}{\norm{u}_1} }^2
=\sigma^2 \smi m_i \norm{p_i-\frac{\frac{\smi u_i q_i}{\norm{u}_1} - \mu }{\sigma}}^2
\\&= \sigma^2  \smi   m_i\norm{p_i-\frac{\frac{\smi u_i (\sigma p_i + \mu )}{\norm{u}_1} - \mu }{\sigma}}^2 
=\sigma^2 \smi  m_i \norm{p_i-\frac{\smi u_i \sigma p_i }{\norm{u}_1\sigma}+ \frac{\smi u_i \mu}{\norm{u}_1 \sigma} -\frac{\mu }{\sigma}}^2
\\&=\sigma^2  \smi  m_i \norm{p_i-\frac{\smi u_i  p_i }{\norm{u}_1}+ \frac{ \mu}{ \sigma} -\frac{\mu }{\sigma}}^2 
=\sigma^2 \norm{m}_1 \smi  w_i \norm{p_i-\frac{\smi u_i  p_i }{\norm{u}_1}}^2,
\end{split}
\end{equation}
where the first equality holds by~\eqref{keyobs}, the second holds by the definition of $p_i$ for every $i\in [n]$, the third and fourth are just a rearrangements, and the last holds by the definition of $w$.

Since $(P,u)$ is a weak $\eps$-coreset for $(P,w)$, we get that
\begin{equation}\label{backtoq}
\begin{split}
\sigma^2 \norm{m}_1 \smi  w_i \norm{p_i-\frac{\smi u_i  p_i }{\norm{u}_1}}^2 
&\leq \sigma^2 \norm{m}_1 (1+\eps)\smi w_i \norm{p_i- \vec{0}}^2 
\\&=  \sigma^2  (1+\eps)\smi m_i \norm{p_i- \vec{0}}^2
=    (1+\eps)\smi m_i \norm{q_i- \mu}^2,
\end{split}
\end{equation}
where the inequality holds since the mean of $(P,w)$ is $\vec{0}$, the first equality holds since $w=\frac{m}{\norm{m}_1} $, and the second holds by~\eqref{keyobs}. Hence, combining~\eqref{reductiontop} with~\eqref{backtoq} proves the lemma.
\end{proof}

%\subsection{Strong Coreset Problem Reduction}
\subsection{Strong Coreset for a Normalized Weighted Set}\label{sec:strong-coreset-red}

Given a normalized weighted set $(P,w)$ as in Definition~\ref{defNormalize}, in the following lemma we prove that a weighted set $(P,u)$ is a strong $(\eps,\delta)$-coreset for $(P,w)$ if some three properties related to the mean, variance, and weights of $(P,u)$ hold with probability at least $1-\delta$.
\begin{lemma}\label{strong_coreset_lema}
Let $(P,w)$ be a normalized weighted set  of $n$ points in $\REAL^d$, $\eps,\delta \in (0,1)$, and $u\in \REAL^n$ such that with probability at least $1-\delta$,
\begin{enumerate}
  \item $\norm {\smi u_i p_i} \leq  \eps $,\label{1}
  \item $\abs{1- \smi u_i } \leq \eps$, and \label{2}
  \item $\abs{ 1- \smi u_i \cdot \norm {p_i}^2} \leq \eps$. \label{3}
\end{enumerate}				
Then, $(P,u)$ is a strong $(2\eps,\delta)$-coreset for $(P,w)$, i.e., with probability at least $1-\delta$, for every $x\in\REAL^d$ we have that	
\begin{align}
\bigg| \smi (w_i - u_i ) \norm{p_i-x}^2  \bigg| \leq  
%2\eps (1+ \norm{x}^2) = 
2\eps \smi w_i\norm{p_i-x}^2.  \label{res}
\end{align}
\end{lemma}

\begin{proof}
First we have that,
\begin{align}
\smi w_i \norm{p_i-x}^2 =\smi w_i\norm{p_i}^2  -2x^T\smi w_i p_i + \norm{x}^2\smi w_i = 1+\norm{x}^2, \label{1+x2}
\end{align}
where the last equality holds by the attributes~\eqref{a}--\eqref{c} of the normalized weighted set $(P,w)$. By rearranging the left hand side of~\eqref{res} we get,
\begin{align}
&\left|\smi (w_i-u_i)\norm{p_i - x}^2 \right| = \left |\smi (w_i-u_i) (\norm{p_i}^2 - 2p_i^Tx + \norm{x}^2 )\right| \\
&\leq \left|  \smi (w_i-u_i)\norm{p_i}^2   \right| + \left|   \norm{x}^2 \smi (w_i-u_i)  \right|  + \left|  2x^T\smi (w_i-u_i)p_i \right| \label{triangle inequality}\\
&=\left| 1- \smi u_i\norm{p_i}^2   \right| +\norm{x}^2  \left| 1- \smi u_i \right|  + \left|  2x^T\smi u_ip_i \right| \label{w_i is distribution and rerranging} \\
&\leq \eps +\eps\norm{x}^2 + 2\norm{x}  \norm{ \smi u_ip_i } \label{almost done},
\end{align}
where~\eqref{triangle inequality} holds by the triangle inequality,~\eqref{w_i is distribution and rerranging} holds by attributes~\eqref{a}--\eqref{c}, and~\eqref{almost done} holds by combining assumptions~\eqref{2},~\eqref{3}, and the Cauchy-Schwarz inequality respectively. We also have for every $a,b\geq 0$ that $2ab \leq a^2 + b^2$, hence,
\begin{align}
2ab = 2\sqrt{\eps}a\frac{b}{\sqrt{\eps}} \leq \eps a^2 + \frac{b^2}{\eps}. \label{2ab bound}
\end{align}
By~\eqref{2ab bound} and assumption~\eqref{1} we get that,
\begin{align}
 2\norm{x}  \norm{ \smi u_ip_i } \leq \eps \norm{x}^2 + \frac{\norm{ \smi u_ip_i }^2}{\eps} \leq \eps \norm{x}^2  +  \frac{\eps^2}{\eps} =\eps \norm{x}^2  + \eps. \label{eps(x +4)}
\end{align}
Lemma~\ref{strong_coreset_lema} now holds by plugging~\eqref{eps(x +4)} in~\eqref{almost done} as,
\begin{align}
\bigg|\smi (w_i-u_i)\norm{p_i - x}^2 \bigg|  &\leq \eps +\eps\norm{x}^2 + \eps \norm{x}^2  + \eps = 2\eps +2\eps \norm{x}^2\\
&  = 2\eps(1+ \norm{x}^2) = 2\eps \smi {w_i \norm{p_i-x}^2} \label{eqFinal},
\end{align}
where the last equality holds by~\eqref{1+x2}.

Observe that if assumptions~\eqref{1},~\eqref{2} and~\eqref{3} hold with probability at least $1-\delta$, then~\eqref{eqFinal} hold also with probability $1-\delta$. We therefore obtain an $(2\eps,\delta)$-coreset.
\end{proof}

%\subsection{Weak Coreset Problem Reduction}
\subsection{Weak Coreset for a Normalized Weighted Set}\label{sec:weak-coreset-red}
Given a normalized weighted set $(P,w)$ as in Definition~\ref{defNormalize}, in the following lemma we prove that a weighted set $(P,u)$ is a weak $(\eps,\delta)$-coreset for $(P,w)$ if and only if with probability at least $1-\delta$ the squared $\ell_2$-norm of the weighted mean of $(P,u)$ is smaller that $\eps$.

\begin{lemma}\label{weak_coreset_lema}
Let $(P,w)$ be a normalized weighted set  of $n$ points in $\REAL^d$, $\eps\in (0,1)$, and $u\in \REAL^n$ be a weight vector.
Let $\displaystyle{\overline{p} =\smi w_i p_i}$ and $\displaystyle{\overline{s} = \smi \frac{u_i}{\norm{u}_1}p_i}$. Then, $(P,u)$ is a weak $\eps$-coreset for $(P,w)$, i.e.,
\[ 
\smi w_i \norm{p_i-\overline{s}}^2 \leq (1+ \varepsilon)\smi w_i \norm{p_i-\overline{p}}^2
\]
if and only if 	
\[ 
\norm{\overline{s}}^2  \leq\varepsilon.  
\]\label{first}
\end{lemma}
\begin{proof}
Observe that $\overline{p}$ is the weighted mean of the points in $P$, since it minimizes the sum of the squared distances from the points in $P$ to it, thus,
\begin{align}\label{mup}
&\smi w_i\norm{p_i - \overline{p}}^2 =\smi w_i \norm{p_i -  \sum_{j =1}^n w_j  p_j }^2 = \smi w_i \norm{p_i-0}^2 = 1,
\end{align}
where the second equality holds by Assumption~\ref{b} of a normalized weighted set and the last holds by Assumption~\ref{c} of a normalized weighted set. We also have that,
\begin{align}\label{mus}
&\smi w_i \norm{p_i-\overline{s}}^2 =  \smi w_i \norm{p_i}^2  -2\overline{s}^T\smi w_i  p_i +\norm{\overline{s}}^2\smi  w_i = 1+ \norm{\overline{s}}^2,
\end{align}
where the last equality holds by Assumptions~\ref{a}--\ref{c} of a normalized weighted set.
%Thus it suffices to prove that,  $ 1+ \norm{\overline{s}}^2  \leq (1 + \varepsilon)$ if and only if $\norm{\overline{s}}^2  \leq\varepsilon$.
Using~\eqref{mup} and~\eqref{mus} we finish the proof by looking at the following two cases:

if $\norm{\overline{s}}^2 >\varepsilon$ then
 $$\smi w_i \norm{p_i-\overline{s}}^2 =1+ \norm{\overline{s}}^2 > 1+\eps > (1+\eps) \smi w_i\norm{p_i - \overline{p}}^2,$$
%where the equality holds by~\eqref{mus},  and the last inequality holds by~\eqref{mup}.

if $\norm{\overline{s}}^2 \leq\varepsilon$ then
 $$\smi w_i \norm{p_i-\overline{s}}^2 =1+ \norm{\overline{s}}^2 \leq 1+\eps \leq (1+\eps) \smi w_i\norm{p_i - \overline{p}}^2.$$
%where the equality holds by~\eqref{mus},  and the last inequality holds by~\eqref{mup}.
\end{proof}

\subsection{From Strong to Weak Coreset Constructions} \label{sec:strongToWeak}
The following lemma proves that any strong $\sqrt{\eps}$-coreset for the mean problem is also a week $\eps$-coreset for the mean problem. 
\begin{lemma}\label{strongToWeek}
Let $(P,w)$ be a normalized weighted set of $n$ points in $\REAL^d$, $\eps\in (0,\frac{1}{36})$ and let $(P,u)$ be a strong $\sqrt{\eps}$-coreset for $(P,w)$. Then $(P,u)$ is also a weak $(36\eps)$-coreset for $(P,w)$, i.e.,
\begin{equation} \label{eqtoProve}
\smi w_i \norm{p_i-\bar{s}}^2 \leq (1+36\eps)\min_{x\in\REAL^d} \smi w_i\norm{p_i-x}^2,
\end{equation}
where $\bar{s} = \smi \frac{u_i}{\norm{u}_1}p_i$ is the weighted mean of $(P,u)$.
\end{lemma}
\begin{proof}
First, observe that if $\norm{\bar{s}} = 0$, then by Lemma~\ref{weak_coreset_lema}, \eqref{eqtoProve} holds immediately. We therefore assume that $\norm{\bar{s}} \neq 0$.

Since $(P,u)$ is a strong $\sqrt{\eps}$-coreset for $(P,w)$, for every $x\in \REAL^d$ we have that
\begin{equation} \label{eq:prop1}
\left|\smi w_i\cdot \norm{p_i-x}^2 - \smi u_i\cdot \norm{p_i-x}^2\right| \leq \sqrt{\varepsilon} \smi w_i\cdot \norm{p_i-x}^2,
\end{equation}
and
\begin{equation}\label{eq:forEveryx}
\smi w_i\cdot \norm{p_i-x}^2 = \smi w_i\norm{p_i}^2 -2x^T\smi w_ip_i +\smi w_i \norm{x}^2 = 1+\norm{x}^2,
\end{equation}
where the last equality holds by the properties of $(P,w)$ in Definition~\ref{defNormalize}.

Therefore, for every $x\in \REAL^d$ we have that
\begin{equation}\label{eq:forEveryx2}
\begin{split}
& \left|\smi w_i\cdot \norm{p_i-x}^2 - \smi u_i\cdot \norm{p_i-x}^2\right| = \left|1+\norm{x}^2 - \smi u_i\cdot \norm{p_i-x}^2\right|\\
& = \left|1+\norm{x}^2 - \smi u_i \norm{p_i}^2 +2x^T \smi u_i p_i -\smi u_i\norm{x}^2\right|,
\end{split}
\end{equation}
where the first equality is by~\eqref{eq:forEveryx}.

Combining~\eqref{eq:prop1},~\eqref{eq:forEveryx} and~\eqref{eq:forEveryx2} yields that for every $x\in \REAL^d$ the following holds
\begin{equation}\label{eq:forEveryx3}
\left|1+\norm{x}^2 - \smi u_i \norm{p_i}^2 +2x^T \smi u_i p_i -\smi u_i\norm{x}^2\right| \leq \sqrt{\eps}(1+\norm{x}^2).
\end{equation}

We now prove that $\norm{\smi u_ip_i} \leq 6\sqrt{\eps}$ using the following case analysis: \textbf{Case (i): }$d=1$, and \textbf{Case (ii): }$d \geq 2$.

\textbf{Case (i): } $d=1$. Plugging $x = 0$ in~\eqref{eq:forEveryx3} yields
\begin{equation}\label{eqx0}
\left|1 - \smi u_i p_i^2 \right| \leq \sqrt{\eps}.
\end{equation}
Plugging $x=1$ in~\eqref{eq:forEveryx3} and combining with~\eqref{eqx0} yields
\begin{equation} \label{eqx1}
\abs{1+2\smi u_i p_i-\smi u_i} \leq 3\sqrt{\eps}.
\end{equation}
Plugging $x=-1$ in~\eqref{eq:forEveryx3} and combining with~\eqref{eqx0} yields
\begin{equation} \label{eqx-1}
\abs{1-2\smi u_i p_i-\smi u_i} \leq 3\sqrt{\eps}.
\end{equation}
Combining~\eqref{eqx1} and~\eqref{eqx-1} implies that
\begin{enumerate}
    \item $\abs{\smi u_ip_i} \leq 3\sqrt{\eps},$ \label{need1}
    \item and $\abs{1-1\smi u_i} \leq 3\sqrt{\eps}.$\label{need11}
\end{enumerate}
Hence, Combining~\eqref{need1} and~\eqref{need11} proves Case (i) as
\[
\smi \frac{u_i}{\norm{u}_1}p_i=\abs{\frac{\smi u_ip_i}{\smi u_i}}\leq \frac{3\sqrt{\eps}}{1-3\sqrt{\eps}}\leq \frac{3\sqrt{\eps}}{1/2} = 6\sqrt{\eps},
\]
where the second inequality is since $\eps\in(0,\frac{1}{36})$.

\textbf{Case (ii): }$d\geq 2$.
We prove Case (ii) by proving the following $3$ properties
\begin{enumerate}[(a)]
  \item $\left|1 - \smi u_i\cdot \norm{p_i}^2\right| \leq \sqrt{\eps}$ \label{toProve1}
  \item $\left|1 -\smi u_i\right| \leq 3\sqrt{\eps}$ \label{toProve2}
  \item $\norm{\smi u_ip_i} \leq 3\sqrt{\eps}$ \label{toProve3}
\end{enumerate}

\textbf{Proof of~\eqref{toProve1}: }This step holds immediately by plugging $x = \textbf{0}_d$ in~\eqref{eq:forEveryx3}.

\textbf{Proof of~\eqref{toProve2}: }Let $s^\bot \in \REAL^d$ be an arbitrary vector that is perpendicular to $\bar{s}$ and let $y = \frac{s^\bot}{\norm{s^\bot}}$. Such a vector $s^\bot$ exists due to our assumption that $\norm{\bar{s}} \neq 0$. We now have that
\begin{equation}\label{eshieshi}
\left|\left(1 - \smi u_i \norm{p_i}^2\right) + \left(1 -\smi u_i\right)\right| = \left|1+\norm{y}^2 - \smi u_i \norm{p_i}^2 +2y^T \smi u_i p_i -\smi u_i\norm{y}^2\right| \leq 2\sqrt{\eps},
\end{equation}
where the first derivation holds by combining that $y$ is perpendicular to $\smi u_i p_i$ (by definition) and that $\norm{y}^2=1$, and the second derivation holds by plugging $x=y$ in~\eqref{eq:forEveryx3}.

Combining~\eqref{eshieshi} with Property~\eqref{toProve1} proves Property~\eqref{toProve2} as
\[
-3\sqrt{\eps} \leq \left(1 -\smi u_i\right) \leq 3\sqrt{\eps}.
\]

\textbf{Proof of~\eqref{toProve3}: }Let $z = \frac{\bar{s}}{\norm{\bar{s}}} = \frac{\smi u_i p_i}{\norm{\smi u_i p_i}}$. We now have that
\begin{align}
& \left|\left(1 - \smi u_i \norm{p_i}^2\right) + \left(1-\smi u_i\norm{z}^2\right) +2\norm{\smi u_i p_i} \right| \nonumber\\
&= \left|\left(1 - \smi u_i \norm{p_i}^2\right) + \left(1-\smi u_i\norm{z}^2\right) +2z^T \smi u_i p_i \right| \label{defZ}\\
& =\left|1+\norm{z}^2 - \smi u_i \norm{p_i}^2 +2z^T \smi u_i p_i -\smi u_i\norm{z}^2\right|\label{defZ2}
\leq 2\sqrt{\eps},
\end{align}
where~\eqref{defZ} holds by the definition of $z$, the first derivation in \eqref{defZ2} holds since $\norm{z}=1$, and the second derivation in~\eqref{defZ2} holds by plugging $x=z$ in~\eqref{eq:forEveryx2}.
Therefore,
\begin{equation}\label{eqAlmostThere}
\left|\left(1 - \smi u_i \norm{p_i}^2\right) + \left(1-\smi u_i\norm{z}^2\right) +2\norm{\smi u_i p_i} \right| \leq 2\sqrt{\eps}.
\end{equation}

Property~\ref{toProve3} now holds by combining~\eqref{eqAlmostThere} with Properties~\ref{toProve1}--\ref{toProve2} as
\[
2\norm{\smi u_i p_i} \leq 6\sqrt{\eps}.
\]
Hence, combining Property~\ref{toProve2} with~\ref{toProve3} satisfies Case (ii) as
\[
\smi \frac{u_i}{\norm{u}_1}p_i=\norm{\frac{\smi u_i p_i}{\smi u_i}}\leq \frac{3\sqrt{\eps}}{1-3\sqrt{\eps}}\leq 6\sqrt{\eps},
\]
where the last inequality holds since $\eps \in(0,\frac{1}{36})$.

By Case (i) and Case (ii) we have that $\norm{\smi \frac{u_i}{\norm{u}_1}p_i} \leq 6\sqrt{\eps}$ for any $d\geq 1$.
Lemma~\ref{strongToWeek} now holds by substituting $u$ and $\bar{s} = \smi \frac{u_i}{\norm{u}_1}p_i$ in Lemma~\ref{weak_coreset_lema}.
\end{proof}

\section{Strong and Weak $(\eps,\delta)$-Coreset Constructions}\label{sec:strondelta}
%In Lemma~\ref{strong_coreset_lema} we showed which properties suffice in order for a weighted set $(P,u)$ to be a strong coreset for a normalized weighted set $(P,w)$. In this section, we aim to compute such a weighted set which satisfies the required properties.
In this section, we aim to compute strong and weak $(\eps,\delta)$-coreset for a normalized weighted set $(P,w)$.

In Section~\ref{sec:sensitivityCoresets} we present a strong coreset construction result which utilizes the sensitivity sampling framework~\cite{DBLP:journals/corr/BravermanFL16}. We then combine this result with the reduction result from strong to weak coresets (see Section~\ref{sec:strongToWeak}) to obtain a weak coreset construction.

In Section~\ref{sec:bernstein} we utilize the Bernstein inequality to obtain a weak coreset for an input set of points contained inside the unit ball. We then show how to leverage this result in order to compute both a strong coreset, based on non-uniform sampling and reweighting of the points. We then obtain a weak coreset by combining the strong coreset construction result with the reduction from Section~\ref{sec:strongToWeak}.
Those weak and strong coresets are smaller than the ones obtained via the sensitivity framework in Section~\ref{sec:sensitivityCoresets}.

\subsection{Sensitivity Based Coresets} \label{sec:sensitivityCoresets}
We now prove that using a smart reweighting scheme of a normalized weighted input set, we can pick a non-uniform random sample of the input, based on the smart weights, to obtain a strong $\eps$-coreset. This is based on the sensitivity framework suggested in~\cite{DBLP:journals/corr/BravermanFL16} and the sensitivity tight bound from~\cite{tremblay2018determinantal}.

\begin{definition} [\textbf{Definition 4.2 in~\cite{DBLP:journals/corr/BravermanFL16}}] \label{def:querySpace}
Let $(P,w)$ be a weighted set of $n$ points in $\REAL^d$. Let $Q$ be a set of items called queries. Let $f:P\times Q \to \REAL$ be a cost function. The tuple $(P,w,Q,f)$ is called a \emph{query space}.
\end{definition}

\begin{definition} [\textbf{Definition 4.5 in~\cite{DBLP:journals/corr/BravermanFL16}}] \label{def:VC}
For a query space $(P,w,Q,f)$, $q\in Q$ and $r\in [0,\infty)$ we define
\[
\range(q,r) = \br{p\in P \mid w(p)\cdot f(p,q) \leq r}.
\]
The dimension of $(P,w,Q,f)$ is the smallest integer $d'$ such that for every $C\subseteq P$ we have
\[
\left| \br{\range(q,r) \mid q\in Q, r\in [0,\infty)} \right| \leq |C|^{d'}.
\]
\end{definition}

\begin{theorem} [\textbf{Theorem 5.5 in~\cite{DBLP:journals/corr/BravermanFL16}}]\label{braverman}
Let $(P,w,Q,f)$ be a query space; see Definition~\ref{def:querySpace}, where $f$ is a non-negative function. Let $s:P\to [0,\infty)$ such that
\[
\sup_{q\in Q} \frac{w(p)f(p,q)}{\sum_{p\in P} w(p)f(p,q)} \leq s(p),
\]
for every $p\in P$ and $q\in Q$ such that the denominator is non-zero. Let $t = \sum_{p\in P} s(p)$ and let $d'$ be the dimension of the query space $(P,w,Q,f)$; See Definition~\ref{def:VC}. Let $c \geq 1$ be a sufficiently large constant and let $\varepsilon, \delta \in (0,1)$. Let $C$ be a random sample of
\[
|C| \geq \frac{ct}{\varepsilon^2}\left(d'\log{t}+\log{\frac{1}{\delta}}\right)
\]
points from $P$, such that $p$ is sampled with probability $s(p)/t$ for every $p\in P$. Let $u(p) = \frac{t\cdot w(p)}{s(p)|C|}$ for every $p\in C$. Then, with probability at least $1-\delta$, for every $q\in Q$ it holds that
\[
(1-\varepsilon)\sum_{p\in P} w(p)\cdot f(p,q) \leq \sum_{p\in C} u(p)\cdot f(p,q) \leq (1+\varepsilon)\sum_{p\in P} w(p)\cdot f(p,q).
\]
\end{theorem}

\newcommand{\strongepsprobcoreset}{\textsc{Sensitivity-sampling-Coreset}}
\begin{algorithm2e}[ht]
\SetAlgoLined
    \caption{$\strongepsprobcoreset(P,w,\eps,\delta)$\label{alg:prob}}{
\begin{tabbing} \label{alg1}
\textbf{Input: \quad } \= A normalized weigthed set $(P,w)$ of $n\geq 2$ points in $\REAL^d$, such that $w=(\frac{1}{n}, \cdots, \frac{1}{n})$,
							 \\ \> an error parameter $\eps\in (0,1)$, and a probability of failure $\delta\in(0,1)$.\\
\textbf{Output:  } \> A weight vector $u\in[0,1)^n$ of cardinality $\norm{u}_0 \in O(\frac{1}{\varepsilon}\left(d+\log{\frac{1}{\delta}}\right))$ non-zero entries \\ \> that satisfies Lemma~\ref{tightSens} and Lemma~\ref{tightSensWeak} .
\end{tabbing}}
	\For{every $i \in \{1,\cdots,n\}$ \label{LineFor}} {
		$\displaystyle{s_i  :=   \frac{1}{2n}\left(1+\norm{p_i}^2\right)}$ \label{s(p) def}\\
	}
	$c:=$ the constant from Theorem~\ref{braverman}.\\
	$S:=$ a random sample (multi-set) of $|S| \geq \frac{2c}{\varepsilon}\left(d+\log{\frac{1}{\delta}}\right)$ points from $P$ sampled i.i.d from the distribution $s=(s_1,\cdots,s_n)$ \label{Ssampled} \\
	\For{every $i \in \{1,\cdots,n\}$} {
		$u_i := \frac{k_i 2\cdot w_i}{s(p_i)\abs{S}}$, where $k_i = |S\cap p_i|$ is the number of times $p_i$ was sampled for $S$. \label{usamled}
	}
	\Return $u$
\end{algorithm2e}

\begin{lemma}[Strong coreset via sensitivity sampling]\label{tightSens}
Let $(P,w)$ be a normalized weighted set of $n$ points in $\REAL^d$ such that $w=(\frac{1}{n},\cdots,\frac{1}{n})^T$. Let $c \geq 1$ be the constant from Theorem~\ref{braverman} and let $\eps, \delta \in (0,1)$.
Let $u$ be the output of a call to $\strongepsprobcoreset(P,w,\eps^2,\delta)$; see Algorithm~\ref{alg:prob}. Then $(P,u)$ is a strong $(\eps,\delta)$-coreset of cardinality $\norm{u}_0 \in O(\frac{1}{\varepsilon^2}\left(d+\log{\frac{1}{\delta}}\right))$ for $(P,w)$, i.e., with probability at least $1-\delta$, for every $x\in \REAL^d$ we have that
\[
(1-\varepsilon)\smi w_i\cdot \norm{p_i-x}^2 \leq \smi u_i\cdot \norm{p_i-x}^2   \leq (1+\varepsilon) \smi w_i\cdot \norm{p_i-x}^2 .
\]
\end{lemma}
\begin{proof}
Mainly the proof here relies on Lemma D.1 of~\cite{tremblay2018determinantal} which states that for every $j\in[n]$, the sensitivity of the $j$th point is:
$$s(p_j) := \sup_{x\in\REAL^d} \frac{\norm{p_j -x}^2}{\smi \norm{p_i - x}^2} =\frac{1}{n}\left(1+\frac{\norm{p_j}^2}{v}\right),  $$
where $v=\smi 1/n\norm{p_i}^2$, and by our assumption we have that $v=1$. Hence, the total sensitivty is
$$t := \sum_{j=1}^n \sup_{x\in\REAL^d} \frac{\norm{p_j -x}^2}{\smi\norm{p_i - x}^2} = \sum_{j=1}\frac{1}{n}\left(1+{\norm{p_j}^2}\right) =1+\sum_{j=1}^n\frac{1}{n}\norm{p_j}^2 = 2.$$
By Theorem~\ref{braverman}, if we sample $|S| \geq \frac{2c}{\varepsilon^2}\left(d+\log{\frac{1}{\delta}}\right)$ i.i.d points from $P$ according to the distribution $(s(p_1)/t ,\cdots ,s(p_n)/t)$, and define the weights vector $u = (u_1,\cdots,u_n)$ where $u_i := \frac{k_i 2\cdot w_i}{s(p_i)\abs{S}}$ and $k_i = |S\cap p_i|$ is the number of times $p_i$ was sampled for $S$, then $(P,u)$ is a strong $(\eps,\delta)$-coreset for $(P,w)$.

In Line~\ref{s(p) def} we compute the distribution $(s_1, \cdots, s_n) = (s(p_1)/t ,\cdots ,s(p_n)/t)$. In Line~\ref{Ssampled} we sample the set  $S$ as required by Theorem~\ref{braverman}, and then we compute in Line~\ref{usamled} the final weights $u$.
\end{proof}

\begin{lemma}[Weak coreset via sensitivity sampling]\label{tightSensWeak}
Let $(P,w)$ be a normalized weighted set of $n$ points in $\REAL^d$ such that $w=(\frac{1}{n},\cdots,\frac{1}{n})^T$. Let $c \geq 1$ be the constant from Theorem~\ref{braverman} and let $\eps, \delta \in (0,1)$.
Let $u$ be the output of a call to $\strongepsprobcoreset(P,w,{\eps/36},\delta)$; see Algorithm~\ref{alg:prob}. Then $(P,u)$ is a weak $(\eps,\delta)$-coreset of cardinality $\norm{u}_0 \in O\left(\frac{1}{\varepsilon}\left(d+\log{\frac{1}{\delta}}\right)\right)$ for $(P,w)$, i.e., with probability at least $1-\delta$, for every  have that
\[
\smi \frac{1}{n}\norm{p_i-\overline{s}}^2\leq (1+\varepsilon)\min_{x\in\mathbb{R}^d}\smi \frac{1}{n}\norm{p_i-x}^2.
\]
\end{lemma}

\begin{proof}
Lemma~\ref{tightSensWeak} immediately holds by combining Lemma~\ref{tightSens} with Theorem~\ref{strongToWeek}.
\end{proof}

\subsection{Bernstein Inequality for Smaller Coresets. } \label{sec:bernstein}
The following theorem is Theorem 6.1.1 from~\cite{tropp2015introduction}. 

\begin{theorem}[Matrix Bernstein.]\label{berns}
Consider a finite sequence $\{S_k\}$ of independent, random matrices with common dimension $d_1 \times d_2$. Assume that (i)  $E (S_k) = 0,$ and (ii) $\norm{S_k}\leq L $ for each index $k$.

Introduce the random matrix $Z = \sum_{k}S_k$.
Let $v(Z)$ be the matrix variance statistic of the sum:
\begin{align*}
    v(Z) &= \max\br{\norm{E(ZZ^T)},\norm{E(Z^TZ)}}
    \\&= \max\br{\norm{\sum_k E(S_kS_k^T)},\norm{\sum_k E(S_k^TS_k)}}.
\end{align*}
Then
$$E(\norm{Z}) = \sqrt{2v(Z)\log(d_1+d_2)} + \frac{1}{3}L\log(d_1+d_2).$$
Furthermore, for all $t\geq 0$,
$$\pr(\norm{Z} \geq t)\leq (d_1 +d_2) \exp\bigg(\frac{-t^2/2}{v(Z)+ Lt/3}\bigg).$$
\end{theorem}
The following corollary is an immediate result of Theorem~\ref{berns}.

\begin{corollary}[Bounding Points in the Unit Ball via Bernstein Inequality.]\label{pointsinunitbenrs}
 Let $\eps,\delta \in (0,1)$, $(P,w)$ be a set of $n$ points in $\REAL^d$, such that for every $i\in [n]$, $\norm{p_i}\leq 1$, and $\smi w_i =1$. Let $S$ be a sample of $k=\frac{4\log((d+1)/\delta)}{\eps}$ points, chosen i.i.d, where each $p_i\in P$ is sampled with probability $w_i$. Let $\overline{s} = \frac{1}{k}\sum_{s\in S}s$. Then with probability at least $1-\delta$ we have that,
$$\norm{\overline{s}}^2 \leq \eps.$$
\end{corollary}

\begin{proof}
Let $z = \sum_{s\in S}s$, and let $v(z) = \max\br{\norm{\sum^k_{i=1} E(s_is_i^T)},\norm{\sum^k_{i=1} E(s_i^Ts_i)}}.$
First, since for every $p\in P$ we have $\norm{p}\leq1$, we get that $$\norm{\sum^k_{i=1} E(s_i^Ts_i)} \leq \norm{\sum^k_{i=1} 1} = k.$$
Also
\begin{align*}
 \norm{\sum^k_{i=1} E(s_is_i^T)} &= \norm{\sum^k_{i=1} \smi w_i p_ip_i^T} = k\norm{\smi w_i p_ip_i^T} \leq k\smi w_i\norm{ p_ip_i^T}
 \\& =  k\smi w_i \sup_{x\in \REAL^d , \norm{x}=1} \norm{p_ip_i^Tx} = 
  k\smi w_i \norm{p_ip_i^T\frac{p_i}{\norm{p_i}}} 
  \\& =   k\smi w_i  \norm{p_i}^2 
  \leq k\smi w_i  = k,
\end{align*}
where the first derivation holds by the definition of mean, the third holds by the rules of norm, the fourth by the definition of matrix norm, the seventh derivation holds since $||p_i||\leq 1$, and the last holds since $\smi w_i =1$.

Hence $v(z) \leq k$. We are interested in bounding the following probability:
$$\pr\bigg(\norm{\overline{s} }^2 \geq \eps\bigg) .$$
To use Theorem~\ref{berns}, we observe that
\begin{align}\label{touse}
\pr(\norm{\overline{s}}^2 \geq \eps)=
   \pr\bigg(\norm{\sum_{i=1}^k s_i/k }^2 \geq \eps\bigg) = \pr\bigg(\norm{\sum_{i=1}^k s_i }^2 \geq \eps k^2\bigg) = \pr\bigg(\norm{\sum_{i=1}^k s_i} \geq \sqrt{\eps} k\bigg). 
\end{align}

Plugging $d_1 = d, d_2 = 1, L=1, Z=z =\sum_{i=1}^k s_i$, and $t= \sqrt{\eps}k$ in Theorem~\ref{berns} yields
\begin{align}\label{useberns}
    \pr\bigg(\norm{\sum_{i=1}^k s_i} \geq \sqrt{\eps} k\bigg) 
    &\leq (d + 1) \exp\bigg(\frac{-(\sqrt{\eps}k)^2/2}{v(z)+ 1\cdot \sqrt{\eps} k/3}\bigg) 
    =  (d + 1) \exp\bigg(\frac{-\eps k^2/2}{v(z)+ 1\cdot \sqrt{\eps} k/3}\bigg)
    \\& \leq (d + 1) \exp\bigg(\frac{-\eps k^2/2}{k+ 1\cdot \sqrt{\eps} k/3}\bigg)
    = (d + 1) \exp\bigg(\frac{-\eps k/2}{1+ 1\cdot \sqrt{\eps}/3}\bigg)\nonumber
    \\ & \leq  (d + 1) \exp\bigg(\frac{-\eps k}{4}\bigg) ,\nonumber
\end{align}
where the third derivation holds since $v(z) = k$, and the last holds since $\eps < 1.$

Substituting $k=\frac{4\log((d+1)/\delta)}{\eps}$ in~\eqref{useberns}
\begin{align}\label{doneusingberns}
    \pr\bigg(\norm{\sum_{i=1}^k s_i} \geq \sqrt{\eps} k\bigg) 
    \leq  (d + 1) \exp\bigg(-\eps \frac{4\log((d+1)/\delta)}{4\eps}\bigg) = (d+1)\frac{\delta }{d+1} = \delta.
\end{align}

Hence, by combining~\eqref{touse} and~\eqref{doneusingberns} we obtain that

$$\pr(\norm{\overline{s} }^2 < \eps) = 1- \pr(\norm{\overline{s} }^2 \geq \eps) = 1-\pr\bigg(\norm{\sum_{i=1}^k s_i} \geq \sqrt{\eps} k\bigg) \geq 1-\delta.$$

\end{proof}

\newcommand{\strongepscoresetberns}{\textsc{Bernstein-CoreSet}}
\begin{algorithm2e}[ht]
\SetAlgoLined
    \caption{$\strongepscoresetberns(P,w,\eps,\delta)$\label{alg:berns}}{
\begin{tabbing} \label{algbern}
\textbf{Input: \quad } \= A normalized weigthed set $(P,w)$ of $n\geq 2$ points in $\REAL^d$,
							 \\ \>	an error parameter $\eps\in (0,1)$,
							  \\ \>and a probability of failure $\delta \in (0,1)$.\\
\textbf{Output:  } \> A weight vector $u\in[0,\infty)^n$ with $O(\frac{\log(d/\delta)}{\eps})$ non-zero entries that satisfies  \\ \>Theorems~\ref{strong-coreset-theorem-bern} and~\ref{weak-coreset-theorem-bern} .
\end{tabbing}}
	\For{every $i \in \{1,\cdots,n\}$ \label{LineFor}} {
	    $\displaystyle{s_i = \frac{w_i\norm{(p^T_i\mid 1)}^2}{\sum_{j=1}^n w_j\norm{(p_j^T\mid 1)}^2}}$\label{line:defdist}
	}	
	
	$k=\frac{4\log((d+1)/\delta)}{\eps}$ 
	
	$S$ := an i.i.d random sample from $P$ of size $|S|=k$, where every point $p_i\in P$ is sampled with probability $s_i$.
	
	\For{every $i \in \{1,\cdots,n\}$} {
	    
	    $\displaystyle{u_i :=\frac{ 2 c_i} { k\norm{(p_i,1)}^2}}$, \label{w'} where $c_i = |S\cap p_i|$ is the number of times $p_i$ was sampled for $S$.
	}	
	\Return $u$
\end{algorithm2e}

\begin{theorem}[Strong coreset via Bernstein inequality] \label{strong-coreset-theorem-bern} Let $(P,w)$ be a normalized weighted set of $n$ points in $\REAL^d$, $\eps,\delta\in (0,1)$, and let $u=(u_1,\cdots,u_n)\in \REAL^n$ be the output of a call to $\strongepscoresetberns(P,w,\eps^2,\delta)$;  See Algorithm~\ref{algbern}.
Then $u$ has $\norm{u}_0\leq \frac{4\log(d+1/\delta)}{\eps^2}$ non-zero entries and $(P,u)$ is a strong $(2\eps,\delta)$-coreset for $(P,w)$, i.e., with probability at least $1-\delta$, for every $x\in\REAL^d$ we have that
\[
\bigg| \smi (w_i -u_i )\norm{p_i-x}^2 \bigg| \leq  2\eps  \smi w_i\norm{p_i -x}^2.
\]
\end{theorem}

\begin{proof}
%%%%%%%%%%%%%%%%%%%%55
%%%%%%%%%%%%%%%%%%%%%%%%%%%%%%%%%%5555
%%%%%%%%%%%%%%%%%%%%%%%%%%%%%%%%%%%%%%%%%%%%%

For every $i \in\br{1,\cdots, n}$, let $s_i = \frac{w_i\norm{(p^T_i\mid 1)}^2}{\smi w_i\norm{(p_i^T\mid 1)}^2}$, and define the distribution vector $s=(s_1,\cdots,s_n)$. 
Let $I$ be an i.i.d random sample from $\br{1,\cdots,n}$ of size 	$k=\frac{4\log((d+1)/\delta)}{\eps^2}$, where every $i\in \br{1,\cdots,n}$ is sampled with probability $s_i$. 
Finally, for every $i\in \br{1,\cdots,n}$ assign a weight $u_i = \frac{2c_i}{k\norm{(p_i\mid 1)}^2}$, where $c_i$ is the number of times $i$ was sampled for $I$. 

%let $\tilde{p_i} = (p_i\mid 1)$,Let $\tilde{P} = \br{\tilde{p_i}\mid i\in [n]}$
%Let $\tilde{S}$ a set of size $k$, that has the corresponding points from $\tilde{P}$ to the sampled indexes in $I$, i.e., $\tilde{S} = \br{\tilde{p_i}| i \in I}$. 

%The squared euclidean distance from the mean of $(P,u)$ to the mean of $(P,W)$ is

%$$\norm{\smi (w_i - u_i)(p^T_i\mid1) }^2.$$

For every $i \in \br{1,\cdots, n}$, let $p_i'= \frac{(p_i^T\mid 1)}{\norm{(p^T_i\mid 1)}^2}$. Let $P' = \br{p_i'\mid i\in \br{1,\cdots,n}}$.
Let $S'$ be a set of size $\abs{I}$, that has the corresponding points from ${P'}$ to the sampled indexes in $I$, i.e., ${S'} = \br{{p'_i}| i \in I}$.

%Let $S'$ be an i.i.d random sample from $P'$ of size $|S|=k$, where for every $i\in [n]$, $p'_i$ is sampled with probability $s_i$ (Line~\ref{defdist}). 

Observe that 
\begin{enumerate}[(i)]
    \item  for every $i\in [n]$, $\norm{p'_i}\leq1$,
    \item $\smi s_i = \smi \frac{w_i\norm{(p^T_i\mid 1)}^2}{\smi w_i \norm{(p_i^T\mid 1)}^2} =1$.
\end{enumerate}
Hence, by Corollary~\ref{pointsinunitbenrs} we have that with probability at least $1-\delta$
$$\norm{\smi  s_i p_i' - \frac{1}{k}\sum_{p'\in S'}p'}^2 \leq \eps^2.$$
Therefore, with probability at least $1-\delta$,
$$\norm{\smi  s_i p_i' - \frac{1}{k}\sum_{p'\in S'}p'} \leq \eps.$$

Substituting $p_i'= \frac{(p_i^T\mid 1)}{\norm{(p^T_i\mid 1)}^2}$, and $s_i =  \frac{w_i\norm{(p^T_i\mid 1)}^2}{\smi w_i \norm{(p_i^T\mid 1)}^2}$ for every $i\in [n]$ we get
$$\norm{\smi \frac{w_i\norm{(p^T_i\mid 1)}^2}{\smi w_i\norm{(p_i^T\mid 1)}^2} \frac{(p_i^T\mid 1)}{\norm{(p^T_i\mid 1)}^2} - \frac{1}{k}\sum_{p'\in S'}p'} \leq \eps.$$

Rearranging the above
$$\norm{\smi\frac{w_i(p_i^T\mid 1)}{\smi w_i\norm{(p_i^T\mid 1)}^2}  - \frac{1}{k}\sum_{p'\in S'}p'} \leq \eps.$$

Observe that $\smi w_i\norm{(p_i^T\mid 1)}^2 = \smi w_i  \norm{p_i}^2 +  \smi w_i =2.$ Hence, multiplying both side by $2$ yields
$$\norm{\smi {w_i(p_i^T\mid 1)}  - \frac{2}{k}\sum_{p'\in S'}p'} \leq 2\eps.$$

%%%%%%%%%%%%%%%%%%%

For every $i\in [n]$, let $c_i$ be the number of times $p'_i$ was sampled for $S'$. By observing that $\frac{1}{k}\sum_{p'\in S'}p' = \smi \frac{c_i}{k}p'_i = \smi \frac{c_i}{k\norm{(p_i^T\mid 1)}^2}(p^T_i\mid1)$, we obtain
\begin{align}\label{neededinequalityviaberns}
\norm{\smi {w_i(p_i^T\mid 1)}  - \smi \frac{2c_i}{k\norm{(p_i^T\mid 1)}^2}(p^T_i\mid1) } \leq 2\eps.
\end{align}

Now, for every $i\in \br{1,\cdots,n}$, let $\tilde{p_i} = (p_i\mid 1)$. Let $\tilde{P} = \br{\tilde{p_i}\mid i\in \br{1,\cdots,n}}$, and let $\tilde{S}$ be a set of size $\abs{I}$, that has the corresponding points from $\tilde{P}$ to the sampled indexes in $I$, i.e., $\tilde{S} = \br{\tilde{p_i}| i \in I}$. 
Finally, recall the weights vector $u=(u_1,\cdots,u_n)$.
The squared euclidean distance from the mean of $(\tilde{P},u)$ to the mean of $(\tilde{P},w)$ is
\begin{align}\label{neededinequalityviaberns2}
\norm{\smi {w_i(p_i^T\mid 1)}  - \smi u_i(p^T_i\mid1) }= \norm{\smi {w_i(p_i^T\mid 1)}  - \smi \frac{2c_i}{k\norm{(p_i^T\mid 1)}^2}(p^T_i\mid1) }\leq 2\eps.
\end{align}
where, the last equality holds by~\ref{neededinequalityviaberns}.

In our algorithm, we sample according to the same distribution, and assign the same weights.
Hence, by~\eqref{neededinequalityviaberns2} we have that \begin{enumerate}
    \item  $ \norm{\smi {w_ip_i}  - \smi \frac{2c_i}{k\norm{(p^T_i\mid 1)}^2}p_i} \leq 2\eps.$ \label{oneyes}
    \item $|\smi w_i - \smi u_i |\leq 2\eps$, and by plugging $\smi w_i= 1$, we obtain $$|1-\smi u_i|\leq 2\eps$$ \label{twoyes}
    \item $2=2 \smi c_i/k = \smi u_i \norm{(p_i\mid 1)}^2 = \smi u_i\norm{p_i}^2 + \smi u_i$. Hence \label{3yes}
    $$|\smi u_i \norm{p_i}^2 -1| = |2-\smi u_i -1| = |1-\smi u_i| \leq 2\eps. $$ 
\end{enumerate}
Theorem~\ref{strong-coreset-theorem-bern} now holds by combining the above properties~\ref{oneyes}--~\ref{3yes} with Lemma~\ref{strong_coreset_lema}.
\begin{comment}
%%%%%%%%%%%%%%%%%%%%%%%%%%%%%%%5
$$\norm{\smi  s_i p_i' - \frac{1}{k}\sum_{p'\in S'}p'}^2 = \norm{\smi (s_i -\frac{c_i}{k})p'_i}^2$$

$$\norm{\smi \frac{w_i\norm{(p^T_i\mid 1)}^2}{\smi w_i\norm{(p_i^T\mid 1)}^2} \frac{(p_i^T\mid 1)}{\norm{(p^T_i\mid 1)}^2} - \smi \frac{c_i}{k} \frac{(p_i^T\mid 1)}{\norm{(p^T_i\mid 1)}^2}}^2 \leq \eps.$$

Rearranging the above
$$\norm{\smi    \frac{w_i(p_i^T\mid 1)}{\smi w_i\norm{(p_i^T\mid 1)}^2}  - \frac{1}{k}\sum_{p'\in S'}p'}^2 \leq \eps.$$
\end{comment}

\end{proof}

%%%%%%%%%%%%%%%%%%%%%%%%%%%%%%%%%%%%%%%%%%%%%%%%%%%%%%%%%%%%%%%%%%%%%%%%%%%%%%%%%%%%%%%%%%%%%%%%%%%%%%%%%%%%%%%%%%%%%%%%%%%%%%%%%%%%%%%%%%%%%%%%%%%%%%%%%%%%%%%%%%%%%%%%%%%%%%%%%%%%%%%%%%%%%%%%%%%%%%%%%%%%%%%%%%%%%%%%%%%%%%%%%%%%%%%%%%%%%%%%%%%%%%%%%%%%%%%%%%%%%%%%%%%%%%%%%%%%%%%%%%%%%%%%%%%%%%%%%%%%%%%%%%%%%%%%%%%%%%%%%%%%%%%%%%%%%%%%%%%%%%%%%%%%%%%%%%%%%%%%%%%%%%%%%%%%%%%%%%%%%%%%%%%%%%%%%%%%%%%%%%%%%%%%%%%%%%%%%%%%%%%%%%%%%%%%%%%%%%%%%%%%%%%%%%%%%%%%%%%%%%%%%%%%%%%%%%%%%%%%%%%%%%%%%%%%%%%%%%%%%%%%%%%%%%%%%%%%%%%%%%%%%%%%%%%%%%%%%%%%%%%%%%%%%%%%%%%%%%%%%%%%%%%%%%%%%%%%%%%%%%%%%%%%%%%%%%%%%%%%%%%%%%%%%%%%%%%%%%%%%
\begin{theorem}[Weak coreset via Bernstein inequality] \label{weak-coreset-theorem-bern} Let $(P,w)$ be a normalized weighted set of $n$ points in $\REAL^d$, $\eps,\delta\in (0,1)$, and let $u=(u_1,\cdots,u_n)\in \REAL^n$ be the output of a call to $\strongepscoresetberns(P,w,\eps,\delta)$;  See Algorithm~\ref{algbern}.
Then $u$ has $\norm{u}_0\leq \frac{4\log(d+1/\delta)}{\eps}$ non-zero entries and $(P,u)$ is a weak $(\eps,\delta)$-coreset for $(P,w)$.
\end{theorem}
\begin{proof}
By Theorem~\ref{strong-coreset-theorem-bern}, the output of a call to $\strongepscoresetberns(P,w,\eps/144,\delta)$ is a strong $\frac{\sqrt{\eps}}{6}$-coreset for $(P,w)$. By Lemma~\ref{strongToWeek}, a $\frac{\sqrt{\eps}}{6}$-coreset for $(P,w)$ is also a weak $\eps$-coreset for $(P,w)$.
\end{proof}

Observe that Algorithm~\ref{alg:berns} and Algorithm~\ref{alg:prob} differ only in the sampling size of the set $S$. The computed distribution and re-weighting of the sampled points are exactly the same. This difference is due to the following facts: (i) Algorithm~\ref{alg:prob} relies on the generic sensitivity framework while Algorithm~\ref{alg:berns} hinges upon the analysis dedicated for the 1-mean problem, and (ii) Algorithm~\ref{alg:berns} uses the Bernstein equality to compute a coreset, while the hidden inequality used in the sensitivity framework is the Hoeffding inequality. Each inequality may be favorable according to the given scenario at hand.

\section{Deterministic $\eps$-Coreset}\label{sec:detcore}
In the previous section we constructed randomized coresets. We now show how to construct both a deterministic weak $\eps$-coreset and a deterministic strong $\eps$-coreset in Theorem~\ref{weak-coreset-theorem} and Theorem~\ref{strong-coreset-theorem} respectively. 
This is by first constructing a weak coreset for an input set of points contained inside the unit ball. We then show how to leverage this result in order to compute a strong coreset, using the Frank-Wolfe algorithm. We then obtain a weak coreset by combining the strong coreset result with the reduction presented in Section~\ref{sec:strongToWeak}.

We use what we call the measure $C_f$, which was defined in Section~{$2.2$} in~\cite{clarkson2010coresets}; See equality~$(9)$.
For a simplex $S$ and concave function $f$, the quantity $C_f$ is defined as
\begin{align}
C_f := \sup \frac{1}{\alpha^2}( f(x) + (y-x)^T\Delta f(x)-f(y)),\label{C_F}
\end{align}
where the supremum is over every $x$ and $z$ in $S$, and over every $\alpha$ so that  $y =x + \alpha(z-x)$ is also in $S$.  The set of such $\alpha$ includes $[0, 1]$, but $\alpha$ can also be negative.

\begin{theorem}[\textbf{Theorem~$2.2$ from~\cite{clarkson2010coresets}}]
For simplex $S$ and concave function $f$, Algorithm~{1.1} from~\cite{clarkson2010coresets} finds a point $x(k)$ on a $k$-dimensional face of $S$ such that
\[
\frac{f(x*) - f(x(k))}{4C_f} \leq \frac{1}{k + 3},
\]
for $k > 0$, where $f(x*)$ is the optimal value of $f$.

\end{theorem}
\begin{theorem}[\textbf{Coreset for points inside the unit ball\label{thm1}}]
Let $P=\{p_1,\cdots,p_n\}$ be a set of n points in $\REAL^d$ such that $\norm{p_i}\leq 1$ for every $i\in [n]$. Let $w=(w_1,\cdots,w_n)\in[0,1]^n$ be a distribution vector, i.e., $\sum_i w_i=1$ and let $\eps\in(0,1)$.Then there is a distribution vector $\tilde{u}=(\tilde{u}_1,\cdots,\tilde{u}_n)\in[0,1]^n$ with $\norm{\tilde{u}}_0\leq \co/\eps$ non-zero entries such that,
\[
\norm{\smi(w_i-\tilde{u_i})p_i}^2\leq \eps.
\]
\end{theorem}
\begin{proof}
Let $S\subseteq \REAL^n $ be the simplex that is the convex hull of the unit basis vectors of $\REAL^n$,  for every $x=(x_1,\cdots,x_n)\in S$ we define $\displaystyle{f(x) =-\norm{\smi (w_i-x_i)p_i}^2}$. Let $C_f$ be defined for $f$ and $S$ as in~\ref{C_F}.

Let $\tilde{\eps} = \eps/\co$, $\tilde{u}$ be the output of a call to Algorithm~1.1 of~\cite{clarkson2010coresets} with $f$ as input after $k=\ceil{1/\tilde{\eps}}$ iterations, and let $f(x^*)$ be the maximum value of $f$ in $S$. Based on Theorem 2.2~\cite{clarkson2010coresets} we have that $\tilde{u}$ is a point on a $k$-dimensional face of $S$ such that,
\begin{align} \frac{ f(x^*) - f(\tilde{u}) }{4C_f} \leq \frac{1}{k+3}. \label{result of frank wolf}\end{align}
Sine $f(x) \leq 0$ for every $x\in S$ We have that,
$$f(x^*) = f(w) = -\norm{\smi (w_i-w_i)p_i}^2 =0.$$
By equality~$(12)$ at section~$2.2.$ in~\cite{clarkson2010coresets} we see that $C_f\leq diam(AS)^2$ for quadratic problems, while $A$ is the matrix of $d\times n$ such that the $i$-th col of $A$ is the $i$-th point in $P$. We have that,
\begin{align*}
 diam(AS)^2 =\sup_{a,b \in AS}\norm{a-b}_2^2=\sup_{x,y \in S}\norm{Ax-Ay}^2_2
\end{align*}
Observe that $x$ and $y$ are distribution vectors, thus
\begin{align*}
\sup_{x,y \in S}\norm{Ax-Ay}_2^2=\sup_{i,j} \norm{p_i-p_j}^2_2.
 \end{align*}
Since $\norm{p_i}\leq1$ for each $i\in [n]$, we have that,
\begin{align*}
\sup_{i,j} \norm{p_i-p_j}^2_2 \leq 2.
\end{align*}
By substituting $f(\tilde{u})=-\norm{\smi(w_i-\tilde{u_i})p_i}^2$,  $C_f\leq 2$, $k=1/\tilde{\eps}$ and $f(x^*)=0$ in~\eqref{result of frank wolf} we get that,
\begin{align}
\frac{ \norm{\smi(w_i-\tilde{u_i})p_i}^2 } {\co} \leq \frac{1}{1/\tilde{\eps}+3}.
\end{align}
Multiplying both sides of the inequality by $\co$ and rearranging yields,
\begin{align}
\norm{\smi(w_i-\tilde{u_i})p_i}^2  \leq \frac{\co}{1/\tilde{\eps}+3}\leq \frac{\co}{1/\tilde{\eps}}=\co\cdot \tilde{\eps} =\eps ,
\end{align}
and since $\tilde{u}$ is a point on a $k$-dimensional face of $S$, we have that,
$$\norm{\tilde{u}}_0 = k = 1/\tilde{\eps} =\co/\eps .$$
\end{proof}

\paragraph{Overview of Algorithm~\ref{alg:two}. } Algorithm~\ref{alg:two} takes as input a normalized weighted set $(P,w)$ and an error parameter $\eps$, and outputs a coreset which is both a weak $\eps$-coreset and a strong $\sqrt{\eps}$-coreset for $(P,w)$. In Lines~\ref{LineFor}--\ref{w'} we augment the data points and their weights, such that the new points are inside the unit ball. We then apply Theorem~\ref{thm1} to construct a coreset of size $O(1/\eps)$ for the new data points of unit length. In Lines~\ref{LineFor}--\ref{u} we compute the output coreset weights.

To construct a weak $\eps$-coreset we call Algorithm~\ref{alg:two} with the normalized weighted input and the error parameter $\eps$; see Theorem~\ref{weak-coreset-theorem}. To construct a strong $\eps$-coreset we simply call Algorithm~\ref{alg:two} with the normalized weighted input and the error parameter $\eps^2$; see Theorem~\ref{strong-coreset-theorem}.

\newcommand{\strongepscoreset}{\textsc{Frank-Wolfe-CoreSet}}
\begin{algorithm2e}[ht]
\SetAlgoLined
    \caption{$\strongepscoreset(P,w,\eps)$\label{alg:two}}{
\begin{tabbing} \label{alg1}
\textbf{Input: \quad } \= A normalized weigthed set $(P,w)$ of $n\geq 2$ points in $\REAL^d$,
							 \\ \>	and an error parameter $\eps\in (0,1)$.\\
\textbf{Output:  } \> A weight vector $u\in[0,\infty)^n$ with $O(1/\eps)$ non-zero entries that satisfies  \\ \>Theorems~\ref{weak-coreset-theorem} and~\ref{strong-coreset-theorem} .
\end{tabbing}}
	\For{every $i \in \{1,\cdots,n\}$ \label{LineFor}} {
		$\displaystyle{p'_i  := \frac{(p_i,1)}{\norm{(p_i,1)}^2}}$ \label{p'}\\
		$\displaystyle{w'_i :=\frac{ w_i\norm{(p_i,1)}^2}{2}}$ \label{w'}
	}	
	Use Thorem~\ref{thm1} to compute a sparse vector $u$ with $O(1/\eps)$ non-zero entries, such that $$\norm{\smi (w'_i - u_i')p'_i}^2\leq \eps$$\label{define u'}\\
	\For{every $i \in \{1,\cdots,n\}$ \label{LineFor2}} {
		$\displaystyle{u_i = \frac{2 u_i'}{\norm{(p_i,1)}^2}}$ \label{u}
	}
	\Return $u$
\end{algorithm2e}

%%%%%%%%%%%%%%%%%%%% (strong linear Coreset) %%%%%%%%%%%%%%%%%%%%%%%%%%%%%%

\begin{theorem}[Strong deterministic coreset via Frank-Wolfe] \label{strong-coreset-theorem} Let $(P,w)$ be a normalized weighted set of $n$ points in $\REAL^d$, $\eps\in (0,1)$, and let $u=(u_1,\cdots,u_n)\in \REAL^n$ be the output of a call to $\strongepscoreset(P,w,(\frac{\eps}{4})^2)$;  See Algorithm~\ref{alg1}.
Then $u$ has $\norm{u}_0\leq \frac{128}{\eps^2}$ non-zero entries and $(P,u)$ is a strong $\eps$-coreset for $(P,w$), i.e., for every $x\in\REAL^d$ we have that
\[
\bigg| \smi (w_i -u_i )\norm{p_i-x}^2 \bigg| \leq  \eps  \smi w_i\norm{p_i -x}^2.
\]
\end{theorem}
\begin{proof}
Let $\eps'=\frac{\eps}{4}$, let  $p'_i  := \frac{(p_i,1)}{\norm{(p_i,1)}^2}$ and $w'_i :=\frac{ w_i\norm{(p_i,1)}^2}{2}$ for every $i\in [n]$. By the definition of $u'$ at line~\ref{define u'} in Algorithm~\ref{alg1}, and since the algorithm gets ${\eps'}^2$ as input, we have that
\begin{align}
\norm{u'}_0\leq \co/{\eps'}^2 =  \frac{128}{\eps^2}, \label{gar1}
\end{align}
and
\begin{align}
\norm{\smi (w'_i - u_i')p'_i}^2\leq {\eps'}^2 .\label{from alg}
\end{align}
For every $i\in[n]$ let $u_i = \frac{2 u_i'}{\norm{(p_i,1)}^2}$ be defined as at Line~\ref{u} of the algorithm. It immediately follows by the definition of $u=(u_1,\cdots,u_n)$ and~\eqref{gar1} that
\begin{align}
\norm{u}_0\leq \co/{\eps'}^2, \label{gar1used}
\end{align}
also we have
\begin{align}
2{\eps'} &\geq 2\norm{\smi (w'_i - u_i')p'_i} = 2\norm {\smi \frac{ w_i\norm{(p_i,1)}^2 -u_i\norm{(p_i,1)}^2 } {2}\cdot \frac{(p_i,1)}{\norm{(p_i,1)}^2}} \nonumber \\
&= \norm { \smi ( w_i  -u_i) \cdot (p_i,1)}
=\norm { \bigg(  \smi (w_i  -u_i) \cdot p_i  \mid  \smi  ( w_i  -u_i ) \bigg )} \label{before last} \\
&\geq \norm {  \smi (w_i  -u_i)\cdot p_i }, \label{2eps bound}
\end{align}
where the first derivation follows from~\eqref{from alg}, the second holds by the definition of $w'_i$,$u'_i$,$u_i$ and $p'_i$ for every $i\in [n]$, and the last holds since $\norm{(x\mid y)} \geq \norm{x}$ for every $x,y$ such that $x\in \REAL^d$ and $y \in \REAL$.

 By~\eqref{before last} and since $w$ is a distribution vector we also have that
\begin{align}
2{\eps'} \geq  \bigg| \smi  ( w_i  -u_i ) \bigg| = \abs{1-\smi u_i}. \label{bounding 1-sum ui}
\end{align}
By theorem~\ref{thm1}, we have that $u'$ is a distribution vector, which yields,
\begin{align*}
2 =  2\smi u'_i = \smi u_i\norm{(p_i,1)}^2  =  \smi u_i\norm{p_i}^2 + \smi u_i,
\end{align*}
By the above we get that $2-  \smi u_i= \smi u_i\norm{p_i}^2$. Hence,
\begin{align}
\abs{ \smi (w_i-u_i)\norm{p_i}^2 }
= \abs{ \smi w_i \norm{p_i}^2 - (2-  \smi u_i)}
= \abs{1- (2- \smi u_i)}
= \abs{\smi u_i - 1} \leq 2{\eps'} \label{bound on diff var*n}
\end{align}
where the first equality holds since $\smi u_i\norm{p_i}^2 = 2-  \smi u_i$, the second holds since $w$ is a distribution and the last is by~\eqref{bounding 1-sum ui}. Now by ~\eqref{bound on diff var*n}, ~\eqref{bounding 1-sum ui} and~\eqref{2eps bound} we obtain that $u$ satisfies Properties~\eqref{1}--\eqref{3} in Lemma~\ref{strong_coreset_lema}. Hence, Theorem~\ref{strong-coreset-theorem} holds as,
\begin{align}
\abs{\smi (w_i-u_i)\norm{p_i - x}^2 } \leq  4{\eps'} \smi {w_i \norm{p_i-x}^2} = \eps \smi {w_i \norm{p_i-x}^2} .
\end{align}
\end{proof}

%%%%%%%%%%%%%%%%%%weak coreset%%%%%%%%%%%%%%%%%%%%%%%%%%%%%%5

\begin{theorem}[Weak deterministic coreset via Frank-Wolfe] \label{weak-coreset-theorem} Let $(P,w)$ be a normalized weighted set of $n$ points in $\REAL^d$, and $\eps\in (0,1)$. Let $u=(u_1,\cdots,u_n)\in \REAL^n$ be the output of a call to $\strongepscoreset(P,w,\eps/576)$;  See Algorithm~\ref{alg1}, and let $\overline{u} = \smi u_i p_i$. Then, $u$ has $\norm{u}_0\leq \co/\eps$ non-zero entries and $(P,u)$ is a weak $4\eps$-coreset for $(P,w)$, i.e.,
\[
\smi w_i\norm{p_i- \overline{u}}^2 \leq (1+\eps) \min_{x\in \REAL^d}{\smi w_i \norm{p_i -x}^2}.
\]
\end{theorem}

\begin{proof}
By Theorem~\ref{strong-coreset-theorem}, the output of a call to $\strongepscoreset(P,w,\eps/576)$ is a strong $\frac{\sqrt{\eps}}{6}$-coreset  for $(P,w)$. By Lemma~\ref{strongToWeek}, a $\frac{\sqrt{\eps}}{6}$-coreset for $(P,w)$ is also a weak $\eps$-coreset for $(P,w)$.
\end{proof}

\section{Weak Coreset Constructions in Sublinear Time}\label{sec:weaksublinear}
In this section we present two coreset construction result, both of which require sublinear time, which compute a weak $\eps$-coreset for the mean problem. Therefore, we cannot assume that the input is a normalized weighted set. The first result utilizes Chebychev's inequality (see Section~\ref{sec:chebychev}), and the second result utilizes the known median of means result (see Section~\ref{sec:medianOfMeans}). 

\subsection{Weak Coreset via Chebychev's Inequality} \label{sec:chebychev}
In what follows we prove that a uniform random sample $S$ of sufficiently large size yields a weak $\eps$-coreset with high probability. To do so, we first use Chebyshev's inequality to show that, with high probability, the mean of $S$ is small, and then conclude by applying Lemma~\ref{weak_coreset_lema}.

\newcommand{\var}{\sigma}
\newcommand{\smip}{\sum_{p\in P}}
\newcommand{\smis}{\sum_{p\in S}}
\begin{lemma}[Weak coreset via Chebychev inequality]  \label{weak_coreset_markov}
Let $P$ be a set of $n$ points in $\REAL^d$, $\mu = \frac{1}{n}\smip p$, and $\var^2=\frac{1}{n}\smip \norm{p - \mu}^2$.  Let $\eps,\delta \in (0,1)$, and let $S$ be a sample of $m = \frac{1}{\varepsilon\delta}$ points chosen i.i.d uniformly at random from $P$. 
%Let $\overline{s}=\frac{1}{|S|}\sum_{p\in S}p$ denote the mean of $S=\br{s_1,\cdots,s_{\frac{1}{\delta\varepsilon}}}$. 
Then, with probability at least $1-\delta$ we have that
$$\norm{ \frac{1}{m}\smis p - \mu}^2 \leq  \eps \var^2.$$
\end{lemma}
\begin{proof}
For any random variable $X$, we denote by $E(X)$ and $\text{var}(X)$ the expectation and variance of the random variable $X$ respectively. Let $x_i$ denote the random variable that is the $i$th sample for every $i\in [m]$. Since the samples are drawn i.i.d, we have
\begin{equation}
\text{var}\left(\frac{1}{m}\smis p\right) = \sum_{i=1}^{m} \text{var}\left(\frac{x_i}{m}\right) =    m \cdot \text{var}\left(\frac{x_1}{m}\right) = m\left(\frac{\var^2}{m^2}\right) = \frac{\var^2}{m} =  \eps\delta\var^2.
\label{variance}
\end{equation}
For any random variable $X$ and error parameter $\eps' \in (0,1)$, the generalize Chebyshev's inequality~\cite{chen2007new} reads that
\begin{equation}\label{eq:GenCheby}
\Pr (\norm{X - E(X)} \geq \eps') \leq \frac{\text{var}(X)}{(\eps')^2}.
\end{equation}
Substituting $X = \frac{1}{m}\smis p$, $E(X) = \mu$ and $\eps' = \sqrt{\eps}\sigma$ in~\eqref{eq:GenCheby} yields that
\begin{equation}
\Pr\left(\norm{ \frac{1}{m}\smis p - \mu} \geq \sqrt{\eps} \var  \right) \leq \frac{\text{var}(\frac{1}{m}\smis p)}{\var^2\eps
}.
\label{chebchev}
\end{equation}
Combining~\eqref{variance} with~\eqref{chebchev} proves the lemma as:
\begin{equation}
\Pr\left(\norm{ \frac{1}{m}\smis p - \mu}^2 \geq \eps \var^2  \right) \leq \frac{\eps\delta\var^2}{\var^2\eps} =\delta .
\end{equation}
\end{proof}

\subsection{Weak Coreset Via Median Of Means} \label{sec:medianOfMeans}

The following algorithm and theorem show how to compute, in sublinear time, a weak $(\eps,\delta)$-coreset of smaller size, compared to the one in Section~\ref{sec:chebychev}, using the median of means approach.

\newcommand{\size}{\frac{4}{\eps}}
\newcommand{\sizediff}{\frac{\eps}{4}}
\newcommand{\myk}{\floor{3.5\log{\left(\frac{1}{\delta}\right)}} +1}
\newcommand{\mykk}{{3.5\log{({1}/{\delta})}}}
\newcommand{\myn}{\frac{4k}{\eps}}
\newcommand{\weaklogdeltacoreset}{\textsc{Prob-Weak-Coreset}}
\begin{algorithm2e}[ht]
\SetAlgoLined
    \caption{$\weaklogdeltacoreset(P,\eps,\delta)$\label{alg:smallProb}}{
\begin{tabbing}
\textbf{Input: \quad } \= A set $P$ of $n\geq 2$ points in $\REAL^d$,
							 \\ \> an error parameter $\eps\in (0,1)$,
							 \\ \> and a probability parameter $\delta \in (0,1)$\\
\textbf{Output:  } \> A subset $S\subseteq P$ that satisfies Lemma~\ref{weak-probabilistic-coreset-theorem}.
\end{tabbing}}
$k:= \myk$. \label{line:defk}\\
$S:=$ an i.i.d sample of size $\myn$.\\
$\br{S_1, \cdots, S_{k}} :=$ a partition of $S$ into $k$ disjoint subsets, each contains $\size$ points .\label{line:sapledsets} \\
Set $\overline{s}_i:=$ the mean of the $i$'th subset $S_i$ for every $i \in [k]$. \label{line:si}\\
$\displaystyle{i^{*}:=\argmin_{j \in [k]}\sum_{i=1}^{k}\norm{\overline{s}_i -\overline{s}_j }_2}$.  \\ \tcp{$i^{*}$ is the index of the closest subset mean $\overline{s}_i^{*}$ to the geometric median of the set $\br{\overline{s}_1,\cdots,\overline{s}_{k}}$.} \label{closeindex}
\Return $S_{i^{*}}$
\end{algorithm2e}

\begin{lemma} [Weak coreset via median of means] \label{weak-probabilistic-coreset-theorem}
 Let $P$ be a set of $n$ points in $\REAL^d$, $\mu = \frac{1}{n}\smip p$, and $\var^2=\frac{1}{n}\smip \norm{p - \mu}^2$. Let $\eps\in (0,1)$, $\delta \in (0,0.9]$, and let $S_{i^{*}}=\br{s_i,\cdots, s_{\abs{S}}} \subseteq \REAL^{d}$ be the output of a call to $\weaklogdeltacoreset(P,\eps,\delta)$;  See Algorithm~\ref{alg:smallProb}. Then $S\subseteq P$ is of size $|S| = \frac{4}{\eps}$, and with probability at least $1-3\delta$ we have that
 $$  \norm{  \frac{1}{\abs{S}} \sum_{i=1}^{\abs{S}} s_i - \mu }^2 \leq  33 \cdot \eps\var^2.$$
 Furthermore, $S_i^*$ can be computed in $O\left(d\left (\log^2{(\frac{1}{\delta})} + \frac{\log{(\frac{1}{\delta})}}{\eps} \right) \right)$ time.
 \end{lemma}

 \begin{proof}
Let $\br{S_1,\cdots,S_k}$ be a set of $k$ i.i.d sampled subsets each of size $\size$ as defined at Line~\ref{line:sapledsets} of Algorithm~\ref{alg:smallProb}, and let $\overline{s}_i$ be the mean of the $i$th subset $S_i$ as define at Line~\ref{line:si}. Let $\displaystyle{\hat{s}:=\argmin_{x\in \REAL^d}\sum_{i=1}^{k}\norm{\overline{s}_i - x}_2}$ be the geometric median of the set of means $\br{\overline{s}_1,\cdots,\overline{s}_{k}}$.

Using Corollary~$4.1.$ from~\cite{minsker2015geometric} we obtain that
$$\Pr \left(\norm{\hat{s} - \mu} \geq 11\sqrt{\frac {\var^2\log({1.4}/{\delta})}{ \myn } }\right )\leq \delta,$$
from the above we have that
\begin{align}
\Pr \left(\norm{\hat{s}- \mu}^2 \geq 121{\frac {\eps\var^2\log({1.4}/{\delta})}{ 4k } }\right )\leq \delta.\label{paperbound}
\end{align}
%By substituting $k=\myk$ we get
%\begin{align}
%\Pr \left(\norm{\hat{s}}^2 \geq 121{\frac {\sizediff\log({1.4}/{\delta})}{{\myk }  } }\right )\leq \delta. \label{paperbound}
%\end{align}
Note that
 \begin{align}
%\Pr \left(\norm{\hat{s}}^2 \geq 121 \sizediff\right) &= \Pr \left(\norm{\hat{s}}^2 \geq 121 \sizediff \frac{\log(1.4/\delta)}{\log(1.4/\delta)}\right) \\
%&= \Pr \left(\norm{\hat{s}}^2 \geq 121\sizediff{\frac { \floor{ 3.5 \log({e}/{\delta})}}{{\mykk}  } }\right )\\
%& \leq \Pr \left(\norm{\hat{s}}^2 \geq 121{\frac {\sizediff\log({1.4}/{\delta})}{{\floor{\log(1/\delta)} +1}  } }\right ).
\Pr \left(\norm{\hat{s}- \mu}^2 \geq 121{\frac {\eps\var^2\log({1.4}/{\delta})}{ 4k } }\right ) &=
\Pr \left(\norm{\hat{s}- \mu}^2 \geq 30.25 \cdot \eps\var^2{\frac {\log({1.4}/{\delta})}{ \myk } }\right )\label{i1} \\
%&=\Pr \left(\norm{\hat{s}- \mu}^2 \geq 30.25 \eps{\frac {\log({1/\delta}) + \log(1.4)}{ \myk } }\right )\nonumber \\
%&\geq \Pr \left(\norm{\hat{s}- \mu}^2\geq  30.25\eps{\frac {\log({1/\delta}) + \log(1.4)}{ \mykk } }\right )\nonumber \\
&\geq \Pr \left(\norm{\hat{s}- \mu}^2 \geq 31\cdot \eps\var^2 \right ), \label{use}
\end{align}
where~\eqref{i1} holds by substituting $k=\myk$ as in Line~\ref{line:defk} of Algorithm~\ref{alg:smallProb}, and~\eqref{use} holds since $\frac {\log({1.4}/{\delta})}{ \myk }<1$ for every $\delta \leq 0.9$ as we assumed.
Combining~\eqref{use} with~\eqref{paperbound} yields,
\begin{align}\label{medianisclose}
 \Pr \left(\norm{\hat{s}- \mu}^2 \geq 31 \cdot \eps \var^2\right ) \leq \delta.
\end{align}

For every $i\in [k]$, by substituting $S = S_i$, which is of size $\frac{4}{\eps}$, in Lemma~\ref{weak_coreset_markov}, we obtain that
\[
\Pr(\norm{\overline{s}_i -\mu}^2 \geq \eps \var^2) \leq 1/4.
\]
Hence, with probability at least $1- ({1}/{4})^{k}$ there is at least one set $S_j$ such that
$$\norm{\overline{s}_j -\mu}^2 \leq \eps\var^2.$$
By the following inequalities:
$$ ({1}/{4})^{k} =  ({1}/{4})^{\myk} \leq ({1}/{4})^{\log(1/\delta)} = 4^{\log(\delta)} \leq 2^{\log(\delta)}=\delta$$
we get that with probability at least $1-\delta$ there is a set $S_j$ such that
 \begin{align}\label{closetothemean}
\norm{\overline{s}_j -\mu}^2 \leq \eps\var^2.
 \end{align}
Combining~\eqref{closetothemean} with~\eqref{medianisclose} yields that with probability at least $(1-\delta)^2$ the set $S_j$ satisfies that
\begin{align}\label{medianandmean}
\norm{\overline{s}_j - \hat{s}}^2 \leq 32 \eps \var^2 . 
\end{align}
 
Let $f:\REAL^d\to [0,\infty)$ be a function such that $f(x)=\sum_{i=1}^{k}\norm{\overline{s}_i - x }_2$ for every $x\in \REAL^d$. Therefore, by the definitions of $f$ and $\hat{s}$,  $\displaystyle{\hat{s}:=\argmin_{x\in \REAL^d}\sum_{i=1}^{k}\norm{\overline{s}_i - x}_2=\argmin_{x\in \REAL^d}f(x)}$. Observe that $f$ is a convex function since it is a sum over convex functions.
By the convexity of $f$, we get that for every pair of points $p,q \in P$ it holds that:
\begin{align}\label{Imsmart}
    \text{if } f(q) \leq f(p) \text{ then } \norm{q-\hat{s}} \leq \norm{p-\hat{s}}.
\end{align}
Therefore, by the definition of $i^{*}$ at Line~\ref{closeindex} of Algorithm~\ref{alg:smallProb} we get that
\begin{align}\label{Imsmart2}
    {i^{*}} \in \argmin_{i\in [k]}\norm{\overline{s}_{i} - \hat{s}}.
\end{align}
Now by combining~\eqref{medianandmean} with~\eqref{Imsmart2} we have that:
  \begin{align}\label{medianandclosetoit}
\Pr\left(\norm{\overline{s}_{i^{*}} - \hat{s}}^2 \leq 32 \eps  \var^2  \right) \geq  (1-\delta)^2 .
 \end{align}
Combining~\eqref{medianandclosetoit} with~\eqref{medianisclose} and noticing the following inequality
 $$(1-\delta)^3 = (1-2\delta +\delta^{2})(1-\delta)  \geq  (1-2\delta )(1-\delta) =1 -\delta -2\delta + 2\delta^2 \geq 1-3\delta,$$
 satisfies Lemma~\ref{weak-probabilistic-coreset-theorem} as,
 \[
 \Pr\left( \norm{  \overline{s}_{i^{*}}  -\mu  }^2 \leq  33\eps  \var^2\right) \leq 1-3\delta.
 \]
 \noindent\textbf{Running time.}
 It takes $O\left( \frac{d\log{(\frac{1}{\delta})}}{\eps}  \right)$ to compute the set of means at Line~\ref{line:si}, and $O\left(d\log{(\frac{1}{\delta})}^2  \right)$ time to compute Line~\ref{closeindex} by simple exhaustive search over all the means. Hence, the total running time is $O\left(d\left (\log{(\frac{1}{\delta})}^2 + \frac{\log{(\frac{1}{\delta})}}{\eps} \right) \right)$.
\end{proof}

\clearpage
\bibliographystyle{alpha}
\bibliography{main}

\end{document}